%% bare_jrnl_compsoc.tex

\documentclass[10pt,journal,compsoc]{IEEEtran}
\usepackage{amsthm}
\usepackage{amsmath}
\usepackage{url}
\usepackage{amssymb}
\usepackage{graphicx}
\usepackage{caption}
\usepackage{color}
\usepackage{array}
\usepackage{booktabs}

\usepackage{bbold}
\usepackage{algorithm}
\usepackage[linesnumbered,algo2e]{algorithm2e}
\usepackage{algpseudocode}
\newtheorem{defn}{Definition}[section]

\usepackage{url}
\usepackage{multirow}
\usepackage{color}
\usepackage{colortbl}

\newcolumntype{L}[1]{>{\raggedright\let\newline\\\arraybackslash\hspace{0pt}}m{#1}}
\newcolumntype{C}[1]{>{\centering\let\newline\\\arraybackslash\hspace{0pt}}m{#1}}
\newcolumntype{R}[1]{>{\raggedleft\let\newline\\\arraybackslash\hspace{0pt}}m{#1}}

\newtheorem{theorem}{Theorem}[section]

 \usepackage{amsmath}

\newcommand{\A}{\mathcal{A}}
\newcommand{\K}{\mathcal{K}}
\newcommand{\F}{\mathcal{F}}

\newcommand{\B}{\mathcal{B}}
\renewcommand{\L}{\mathcal{L}}
\renewcommand{\S}{\mathcal{S}}
\newcommand{\Y}{\mathcal{Y}}
\newcommand{\W}{\mathcal{W}}
\newcommand{\D}{\mathcal{D}}

\newcommand{\name}{$\mathsf{EC3}$}
\newcommand{\namei}{$\mathsf{iEC3}$}
\usepackage{xcolor}

\ifCLASSOPTIONcompsoc
  \usepackage[nocompress]{cite}
\else
  \usepackage{cite}
\fi
\ifCLASSINFOpdf
\else
\fi
\begin{document}
\title{\name: Combining Clustering and Classification for Ensemble Learning}

\author{Tanmoy Chakraborty
%, Vijay Balakrishnan~\IEEEmembership{}

\IEEEcompsocitemizethanks{\IEEEcompsocthanksitem T. Chakraborty is with the Department
of Computer Science \& Engineering, Indraprastha Institute of Information Technology Delhi (IIIT Delhi), India. 
%V. BalaKrishnan is with the Department of Computer Science, University of Maryland, College Park, USA.
\protect\\
E-mail: tanmoy@iiitd.ac.in%vijaym.balakrishnan@gmail.com.
}% <-this % stops an unwanted space
\thanks{}}


\markboth{Journal of \LaTeX\ Class Files,~Vol.~13, No.~9, September~2014}%
{Shell \MakeLowercase{\textit{et al.}}: Bare Demo of IEEEtran.cls for Computer Society Journals}

%\IEEEpubid{\makebox[\columnwidth]{\hfill 0000--0000/00/\$00.00~\copyright~2014 IEEE}%


\IEEEtitleabstractindextext{%
\begin{abstract}
Classification and clustering algorithms have been proved to be successful individually in different contexts. Both of them have their own advantages and  limitations. For instance, although classification algorithms are more powerful than clustering methods in predicting class labels of objects, they do not perform well when there is a lack of sufficient manually labeled reliable data. On the other hand, although clustering algorithms do not produce label information for objects, they provide supplementary constraints (e.g., if two objects are clustered together, it is more likely that the same label is assigned to both of them)  that one can leverage for label prediction of a set of unknown objects. Therefore, systematic utilization of both these types of algorithms together can lead to better prediction performance. In this paper, We propose a novel algorithm, called \name~ that merges classification and clustering together in order to support both binary and multi-class classification. \name~ is based on a principled combination of multiple classification and multiple clustering methods using an  optimization function. We theoretically show the convexity and  optimality of the problem and solve it by block coordinate descent method. We additionally propose \namei, a variant of \name~ that handles {\em imbalanced dataset}. We perform an extensive experimental analysis by comparing \name~ and \namei~ with $14$ baseline methods ($7$ well-known standalone classifiers, $5$ homogeneous ensemble classifiers, and $2$ heterogeneous ensemble classifiers that merge classification and clustering) on $13$ standard benchmark datasets. We show that our methods outperform other baselines for every single dataset, achieving at most $10\%$ higher AUC. Moreover our methods are {\em faster} (1.21 times faster than the best heterogeneous baseline), more {\em resilient to noise} and {\em class imbalance} than the best  baseline method.
\end{abstract}

}

\maketitle

\IEEEdisplaynontitleabstractindextext

% \fi
\IEEEpeerreviewmaketitle

\IEEEraisesectionheading{\section{Introduction}\label{sec:introduction}}

\IEEEPARstart{S}{uppose} in a classification task, direct access to the raw data (i.e., objects and their features) is not allowed due to privacy and storage issues. Instead, only the opinions of other experts on the actual class of the objects are available. The experts might be human annotators or different predictive models. The task is to better predict the class of each object. 
This scenario is realistic -- e.g., in financial sector, customers might have accounts across different banks. To analyze customer segmentation or to detect fraud accounts, it may be unsafe to provide customer information across different banks to the third-party experts. Instead, one can conduct a preliminary analysis/prediction at each bank individually and then combine the predictions from different banks to obtain an aggregated result.

Ensembles have already  proven  successful in both unsupervised \cite{Dimitriadou2001,Strehl:2003,FernB04} and supervised  \cite{bagging,Schapire:1999,0002CS17,Gomes:2017} learning. In supervised classification, objects are generally classified one at a time with the assumption that they are drawn from an {\em independent and identical distribution} (i.i.d.); thus the inter-dependencies between objects are not considered \cite{ZhangYZHL14}.
Moreover, with less amount of labeled data it may be hard to reliability predict the label of an unknown object. On the other hand, unsupervised clustering methods complement it by considering object-relationships, thus providing supplementary constraints in classifying objects.  For example, a pair of objects which are close in a feature space are more likely to obtain same class label than those pairs which are far apart from each other. 
These supplementary constraints  can be useful in improving the
generalization capability of the resulting classifier, especially when labeled data is rare \cite{Acharya2011}. Moreover, they can be useful for designing learning methods where there is a significant difference in training and testing data distributions \cite{Acharya2011}.

Recent efforts have shown that combining classification and clustering methods can yield better classification result \cite{NIPS2009_3855,Ao:2014}.  However, the crucial question is how to combine both classification and clustering? Acharya et al. \cite{Acharya2011} suggested a post-processing technique by leveraging clustering results that refines the aggregated output obtained from the classifiers. Later, Gao et al. \cite{NIPS2009_3855} suggested an object-group bipartite graph based model to combine two types of results. Ao et al. \cite{Ao:2014} proposed a complex unconstrained probabilistic embedding method to solve this problem. 

{\bf Our Proposed Ensemble Classifier:}
In this paper, we propose \name, an {\bf E}nsemble {\bf C}lassifier that  combines both {\bf C}lassification and {\bf C}lustering to see if combining supervised and unsupervised models achieves better prediction results. 
%\name~is unsupervised in the sense that it does not require to access the raw objects and their features.
\name~ is built on two fundamental hypotheses -- (i) if two objects are clustered together by multiple clustering methods, they are highly likely to be in the same class, (ii) the final prediction should not deviate much from the majority voting of the classifiers. We ensure that each group (or, class) consists of homogeneous objects (i.e., objects within a group are likely to have same set of features). This in turn ensures that the group characteristics are same as the characteristics of constituent objects inside the group. We map this task into an optimization problem, and prove that the proposed objective function is convex. We use block coordinate descent for optimization. 
One of the motivations behind designing an ensemble clustering and classification approach is that it should be able to handle imbalanced data, i.e., a dataset where majority one or few class labels dominate others and the class distribution is skewed. We observe that \name~tends to be effective for balanced datasets. Therefore, we further design \namei, a variant of \name~particularly to tackle imbalanced datasets. \namei~turns out to  perform as well or even better than \name.

Gao et al. \cite{NIPS2009_3855} classified the entire spectrum of learning methods based on two dimensions: one dimension is the level of supervision (unsupervised, semi-supervised and supervised), and other dimension is the way the ensemble executes (no ensemble, ensemble at the raw data level and ensemble at the output level). In this sense, our proposed framework is a semi-supervised approach which combines multiple data at the output level.

\begin{table}[!t]
\centering
\caption{Summary of the comparative evaluation among UPE \cite{Ao:2014} (best baseline) and \namei~ on the largest dataset (Susy).}\label{summary}
\vspace{-3mm}
\scalebox{0.8}{
\begin{tabular}{|c|c|c|c|c|c|}
\hline
 \multirow{2}{*}{Method}&\multicolumn{2}{c|}{Accuracy} & Class Imbalance & Robustness & Runtime \\ \cline{2-3}
  & AUC & F-Sc & ($30\%$ imbalance) & (10 random solutions) & (in seconds)\\\hline
 UPE & 0.74 & 0.72 & 71\% of original & 58\% of original & 28721\\\hline
 \namei~& {\bf 0.78} & {\bf 0.76} & {\bf 87\%} of original & {\bf 88\%} of original & {\bf 27621}\\\hline
\end{tabular}}
%\vspace{-5mm}
\end{table}

{\bf Summary of our Contributions:} We test our method with 14 baselines, including $7$ standalone   classifiers, $5$ homogeneous ensemble classifiers (that combine multiple classifiers) \cite{Reid2009,bagging,Schapire:1999,Ao:2014}, and $2$ sophisticated heterogeneous ensemble classifiers (that combine both classifiers and clustering methods) \cite{NIPS2009_3855,Ao:2014} on 13  datasets (Section \ref{setup}).  Our method turns out to be superior to other baselines w.r.t following four aspects (a summary of the comparative evaluation for the largest dataset is presented in Table \ref{summary}): 

\noindent(i) {\bf Consistency:} \name~and \namei~outperform other baselines for every single dataset (\namei~is superior to \name). On average, \namei~achieves at least $3\%$ and at most $10\%$ higher accuracy (in terms of Area under the ROC curve) than UPE (the best baseline method) \cite{Ao:2014}. 

\noindent(ii) {\bf Handling class imbalance:} \namei ~efficiently handles datasets where the class size is not uniformly distributed -- with $30\%$ random class imbalance in the largest dataset (Susy), \namei~(UPE) retains 89\% (71\%) of its  accuracy that it achieves with completely balanced data. 

\noindent(iii) {\bf Robustness:} \namei~is remarkably resilient to random noise -- \namei~(UPE) retains at least 88\% (58\%) of its original performance (noise-less scenario) with 10 random solutions (noise) injected to the base set for Susy. 

\noindent(iv) {\bf Scalability:} \namei~is faster than any heterogeneous ensemble model --  on average \namei~is $1.21$ times faster than UPE. 

We also show the effect of the number and the quality of base methods on the performance of our methods. We further show that the runtime of our method is linear in the number of base methods and the number of objects, and quadratic in the number of classes. More importantly, it is faster than two heterogeneous ensemble classifiers \cite{NIPS2009_3855,Ao:2014}. 
In short, \namei~ is a fast, accurate and robust ensemble classifier.

{\bf Organization of the paper:} The paper starts with a comprehensive literature survey in Section \ref{related work}. Section \ref{method} presents a detailed description of our proposed method. Section \ref{setup} shows the experimental setup (the datasets, base and baseline algorithms used in this paper). A detailed experimental results including comparative evaluation (Section \ref{compare}), handling class imbalance (Section \ref{class_imbalance}), robustness analysis (Section \ref{robustness}) and scalability (Section \ref{runtime}) is shown in Section \ref{result}. We conclude  by summarizing the  paper and mentioning possible future directions in Section \ref{conclusion}.

\section{Related Work}\label{related work}
 Major research has been devoted to develop/improve supervised learning algorithms. Decision tree, Support vector Machine, logistic regression, neural network are a few of many such supervised methods \cite{Kotsiantis:2007}. 
 On the other hand, many unsupervised learning algorithms were proposed based on different heuristics how to groups objects in a feature space. Examples include K-Means, hierarchical clustering, Gaussian (EM) clustering DBSCAN etc. \cite{1427769}.    
Gradually, researchers started thinking how to leverage the output of an unsupervised method in a supervised learning, which led to the idea of transduction learning \cite{mepdh2014bis} and semi-supervised learning \cite{4724730}. Existing semi-supervised algorithms are hard to be adopted to our setting since they usually consider one supervised model and one unsupervised model. Goldberg and Zhu  \cite{Goldberg:2006} suggested a graph-based semi-supervised approach to address sentiment analysis of rating inference. However, their approach is also unable to combine multiple supervised and unsupervised sources.

Ensemble learning has been used for both unsupervised and supervised models. \cite{Dimitriadou2001,Strehl:2003,FernB04}  focus on ensemble-based unsupervised clustering. State-of-the-art supervised ensemble methods including Bagging \cite{bagging}, Boosting \cite{Schapire:1999}, XGBoost \cite{Chen:2016:XST},  
rule aggregation \cite{CIS-230716} are derived from diversified base classifiers. Meta-feature generation methods such as Stacking \cite{Reid2009} and adaptive mixture of experts \cite{Jacobs:1991} build a meta-learning method on top of existing base methods by considering the output of base methods as features.
Boosting, rule ensemble \cite{Gomes:2017}, Bayesian averaging model \cite{citeulike:474258} use training data to learn both the base models and how their outputs are combined; whereas methods like Bagging, random forest \cite{Breiman:2001}, random decision tree \cite{1565674}  train base models on the training data and make majority voting to come to a consensus. Therefore, they also need  huge labeled data and can work at the raw data level (unlike ours which works at the output level). Other research involved how to efficiently choose base models which are accurate and as diverse as possible so that the generalized error will be minimized \cite{Kuncheva:2004,Bauer:1999}.
Several methods have tried to incorporate  clustering  into the classification model in a {\em semi-supervised manner} (see  \cite{Zhu05}). Notable methods include SemiBoost \cite{Mallapragada2009}, ASSEMBLE \cite{Bennett:2002}, try-training \cite{Zhou:2005} etc. Their major focus was to learn from a limited amount of labeled data and plenty of unlabeled data. Few of them  considered only one base classifier and one base clustering method, rather than many which we do. Xiao et al. \cite{XIAO201673} created multiple clusters from the training set, and clusters of pair-wise classes are combined to generate multiple training samples. Each classifier is trained on a certain training sample, and a weighted voting approach is used to combine the results. 

Ensemble techniques in the unsupervised learning have mostly focused on combining multiple partitions in order to minimize the disagreement. Since different unsupervised models can produce different number of clusters and there is no label information associated with each cluster, it is hard to decide how to combine multiple clustering outputs. Existing ensemble clustering models differ from each other based on their selection of consensus function and the representation of the base results. Usually the base results are summarized using a graph \cite{FernB04,Strehl:2003} or a multi-dimensional array \cite{Singh:2003}. The model combination is usually performed by mapping the problem into information-theoretic \cite{Strehl:2003}, co-relational  clustering \cite{Gionis:2007} or median partitioning \cite{1565690,LiDJ07} optimization problem.

There have been very limited attempts to combine multiple  base classifiers and clustering methods. C$^3$E \cite{Acharya2011} is one of the early ensemble models that combine heterogeneous base methods. It uses multiple classifiers to generate an initial class-level probability distribution for each object. The distribution is then refined using cluster ensemble.  Gao et al.  proposed the BGCM model \cite{NIPS2009} and its extension \cite{NIPS2009_3855}, which derives an object-group bipartite graph out of the base models by embedding both objects and groups into a fixed dimension. BGCM was reported to outperform C$^3$E \cite{NIPS2009_3855}. The UPE model \cite{Ao:2014} casts this ensemble task as an unconstrained probabilistic embedding problem. It assumes that both objects
and groups (classes/clusters) have latent coordinates without constraints in a $D$-dimensional Euclidean space. Then a mapping from the embedded space into the space of the model yields a probabilistic generative process. It generates the final prediction by computing the distances between the object and the classes in the embedded space. 
%UPE was reported to be better than BGCM \cite{Ao:2014}.

\name~falls into the same group of algorithms which combine multiple supervised and unsupervised methods and work at the meta-output level without accessing the raw data. However, the fundamental mechanism of \name~differs from BGCM and UPE -- we provide a consensus at the object level as well as at the group level and force the groups to be constructed by homogeneous objects (objects with similar features). We propose an objective function to ensure that -- (i) the group characteristics is similar to the characteristics of its constituent objects, (ii) the more two objects are part of same base groups, the higher the probability that they are assigned to the same class,  (iii) class distribution of an object is similar to its average class distribution obtained from multiple base classifiers, and (iv) class distribution of a group is similar to the average class distribution of its constituent objects.   Extensive experiments on 13 different datasets confirm that \name~and \namei~ outperform both BGCM and UPE, and $12$ other baselines for every single dataset (see Section \ref{result}). Note that we do not consider C$^3$E as a baseline since BGCM was already reported to outperform  C$^3$E \cite{NIPS2009}. 

\begin{table}
\caption{Important notations and denotations}\label{tab:notation}
\vspace{-3mm}
\scalebox{1}{
\begin{tabular}{l|l}
\hline
Symbol & Definition\\\hline
$\mathcal{O}$& $\{ O_1,O_2,\cdots,O_N\}$, set of $N$ objects\\
$\mathcal{L}$ & $\{ 1,2,\cdots,l\}$, set of $l$ classes\\
$C_1$ & \# of base classifiers \\
$C_2$ & \# of base clustering methods\\
$G_1$ & \# of groups obtained from $C_1$ base classifiers\\
$G_2$ & \# of groups obtained from $C_2$ base clustering methods\\
$G$ & $G_1+G_2$, total number of base groups\\
$\mathcal{A}^m$ & $ \mathcal{A}^m \in \mathbf{R}_{n\times G}$, object-group membership matrix\\
$\mathcal{A}^c$ & $ \mathcal{A}^c \in \mathbf{R}_{N\times N}$, object-object co-occurrence matrix\\
$\mathcal{F}^g$ & $\mathcal{F}^g \in \mathbf{R}_{G\times l}$, class distribution for groups\\
$\mathcal{F}^o$ & $\mathcal{F}^o \in \mathbf{R}_{N\times l}$, class distribution for objects\\
$L(.)$ & Function returning the class of an object/group.\\
\hline
\end{tabular}}
%\vspace{-5mm}
\end{table}

\begin{figure}[!t]
\centering
\includegraphics[width=0.8\columnwidth]{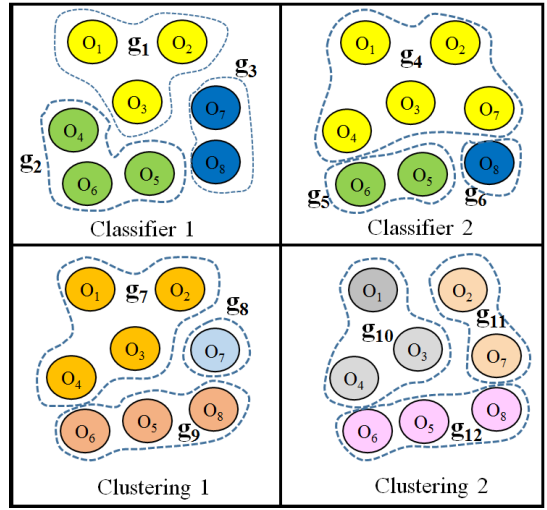}
\caption{(Color online) An illustrative example of the object grouping. There are $8$ objects with $3$ classes (yellow, green, blue).  Two classifiers predict the classes of the objects differently (top row). Two clustering methods group objects differently into three clusters (bottom row). Here, $C_1=2$, $C_2=2$, $G_1=G_2=6$, $G=12$. Note that we cannot label a cluster obtained from a base clustering method with a class. }\label{fig:demo}
\vspace{-5mm}
\end{figure}

\section{Methodology}\label{method}
Suppose we are given $N$ different objects $\mathcal{O}=\{O_1,O_2,\cdots,O_N\}$. We also know that they belong to $l$ different classes $\mathcal{L}=\{1,2,\cdots,l\}$. We are provided the outputs of $C_1$ base classifiers along with $C_2$  base clustering methods. For the sake of simplicity, suppose: (i) each object is assigned to only one class by a classifier and to only one cluster by a clustering method (i.e., disjoint clustering), and (ii) each base clustering method produces $l$  clusters\footnote{The generalization of this assumption is straightforward, and does not violate the solution of the problem.}. Following \cite{NIPS2009_3855,Ao:2014}, we call both ``classes'' and ``clusters'' discovered by base methods as ``base groups''. Therefore, from the base classifiers and base clustering methods we obtain $G_1=C_1\times l$ and $G_2=C_2\times l$ groups respectively, totaling $G=G_1+G_2$ base groups. Figure \ref{fig:demo} presents a toy example with 8 objects, 3 classes, 2 base classifiers and 2 base clustering methods. Each base method produces 3  groups, totaling $12$ base groups. Table \ref{tab:notation} summarizes our notations.  
%Note that we cannot tag a cluster obtained from a base clustering method with a class.  

From the output of the base methods, we can construct the following matrices:
\begin{defn}[Membership Matrix]
We define membership matrix as $\A^m|_{N\times G}$  where $\A^m_{ij}=1$, if object $O_i$ belongs to  group $g_j$, $0$ otherwise.
\end{defn}

\begin{defn}[Co-occurrence Matrix]
We define co-occurrence matrix as $\A^c|_{N\times N}$  where $\A^c_{ij}\in \mathbf{R}$ indicates the number of times two objects $O_i$ and $O_j$ co-occur together in the base groups.
\end{defn}

Moreover, we define two conditional probability matrices:
\begin{defn}[Object-class Matrix]
We define object-class matrix as $\F^o|_{N\times l}$ where each entry $\F^o_{ij}=P(L(O_i)=j|O_i)$ indicates the probability of an object $O_i$ being assigned to class $j$, where the function $L(.)$ returns the class of an object or a group. 
\end{defn}
\begin{defn}[Group-class Matrix]
We define group-class matrix as $\F^g|_{G\times l}$ where each entry $\F^g_{ij}=P(L(g_i)=j|g_i)$ indicates the probability of a group $g_i$ being labeled as class $j$.
\end{defn}

Similarly, we define $\Y^o$ ({\em resp.} $\Y^g$) as an average class distribution matrix corresponding to $\F^o$ ({\em resp.} $\F^g$) such that  $\Y^o_{ij}$ ({\em resp.} $\Y^g_{ij}$) indicates  the fraction of times object $O_i$ is labeled as class $j$ by the base classifiers ({\em resp.} fraction of objects in group $g_i$ labeled as class $j$ by the base classifiers). However, measuring $\Y^g_{ij}$ may not be straightforward because each of the $C_1$ base classifiers may produce a different class for each object. Therefore, we consider $C_1$ different instances of each object to calculate $\Y^g_{ij}$.

Note that both $\A^m$ and $\A^c$ are not normalized\footnote{Normalization is needed to show the convexity of the problem in Theorem \ref{theory:convex}.}. Therefore, we can learn a bi-stochastic matrix  each for $\A^m$ and $\A^c$. Wang et al. \cite{5694009} proposed different ways of generating a bi-stochastic matrix from an adjacency matrix using Bregman divergence. They concluded that Kullback-Leibler (KL) divergence is superior than Euclidean  distance for bi-stochastic matrix generation. Here we also use KL divergence to generate the bi-stochastic matrices using the methods suggested in \cite{5694009} as follows.

Without loss of generality, let us assume that $\A|_{N\times N} \in \{\A^m, \A^c\}$. We intend to generate a bi-stochastic matrix $\K$ that optimally approximates $\A$ in the KL divergence sense by solving the following optimization problem.

\begin{equation}\label{min1}
\begin{aligned}
& \underset{\mathcal{A}}{\text{min}}
& &KL(\K,\A)=\sum_{i,j}^N [\K_{ij}\log \frac{\K_{ij}}{\A_{ij}} - \K_{ij} + \A_{ij} ] \\
& \text{subject to}
& & \K \geq 0,\ \K\cdot \mathbb{1}= \mathbb{1},\ \K=\K^\top
\end{aligned}
\end{equation}
Here $\mathbb{1}$ is an all-ones matrix, and $\K^\top$ is the transpose of $\K$.

\begin{theorem}
The optimization problem (\ref{min1}) is a convex problem.
\end{theorem}
\begin{proof}
The constraints in (1) are all linear; therefore they are convex.
We only need to show the objective function is convex.

We will prove each individual term in (1) as convex. Let $J=\K_{ij}\log \frac{\K_{ij}}{\A_{ij}} - \K_{ij} + \A_{ij}=\K_{ij}\log \K_{ij} - \K_{ij} \log \A_{ij}  - \K_{ij} + \A_{ij}$. Since $\K_{ij}$ is constant, both $\A_{ij}$ and $-\K_{ij}\log\A_{ij}$ are convex. Moreover the second derivative of $\K_{ij}\log \K_{ij}$ is $\nabla^2_{\K_{ij}} \K_{ij} \log_{\K_{ij}}=\frac{1}{\K_{ij}}\geq 0$ (since $\K\geq 0$). Therefore $\K_{ij}\log \K_{ij}$ is also convex. This in turn proves that $J$ (as well as (1)) is convex. 
\end{proof}

\if{0}
\begin{proof}
The constraints in (1) are all linear; therefore they are convex.
We only need to show the objective function is convex.

We will prove each individual term in (1) as convex. Let $J=\K_{ij}\log \frac{\K_{ij}}{\A_{ij}} - \K_{ij} + \A_{ij}=\K_{ij}\log \K_{ij} - \K_{ij} \log \A_{ij}  - \K_{ij} + \A_{ij}$. Since $\K_{ij}$ is constant, both $\A_{ij}$ and $-\K_{ij}\log\A_{ij}$ are convex. Moreover the second derivative of $\K_{ij}\log \K_{ij}$ is $\nabla^2_{\K_{ij}} \K_{ij} \log_{\K_{ij}}=\frac{1}{\K_{ij}}\geq 0$ (since $\K\geq 0$). Therefore $\K_{ij}\log \K_{ij}$ is also convex. This in turn proves that $J$ (as well as (1)) is convex. 
\end{proof}
\fi

This convex problem can be solved by projecting $\A$ onto the constraints as mentioned in \cite{Inderjit} (see the pseudo-code in Algorithm \ref{alg:bi}). Following this, we obtain two bi-stochastic matrices $\K^m$ and $\K^c$ corresponding to $\A^m$ and $\A^c$ respectively.

\begin{algorithm2e}[!t]
\DontPrintSemicolon
\SetAlgoNoEnd
\hrule
\KwIn{Similarity matrix: $\A$, threshold: $\epsilon$}
\KwOut{Bi-stochastic matrix $\K$}
\hrule 
Initialize $\K=0$\\
Initialize $\K'=\A$\\
\While{$||\K- \K'||_F>\epsilon$}{
$\K=\K'$\\
$d_{i}=\sum_{j=1}^N \K'_{ij}, \forall i$\\ $\K_{ij}=\frac{\K'_{ij}}{d_i}, \forall i,j$ \\
$\K_{ij}=\K_{ji}=(\K_{ij}\K_{ji})^{\frac{1}{2}}, \forall i,j$
}
\Return $\K$;
 \hrule
\caption{Generating a bi-stochastic matrix as mentioned in \cite{Inderjit} (here $||.||_F$ is the  Frobenius norm).} \label{alg:bi}
\end{algorithm2e}

\subsection{Objective Function}\label{obj_func}

Our final objective function consists of four components generated by the following hypotheses:\\
{\bf (i) Similarity between a group and its constituent members:} If an object is a part of a group, the class distribution of both the object and the group should be similar. We capture this by the following expression:
\vspace{-2mm}
\begin{equation}
J1=\sum_{i=1}^N \sum_{j=1}^G \K^m_{ij}|| \F^o_{i.}-\F^g_{j.}||^2
\end{equation}
where $\F^o_{i.}$ and $\F^g_{j.}$ are the class distribution vectors of $i^{th}$ object and $j^{th}$ group respectively, and $||.||$ is the 2-norm of a vector.\\

\noindent{\bf (ii) Similarity between two objects inside a group:} The more two objects are assigned to the same groups, the higher the probability that they are in the same class (``co-occurrence principle''). We capture this via the following equation:

\begin{equation}
J2=\sum_{i=1}^N \sum_{j=1}^N  \K^c_{ij} ||\F^o_{i.} - \F^o_{j.}||^2
\end{equation}

\noindent{\bf (iii) Similarity between the object and its  average class distribution:} The final class distribution of an object should be closer to its average class distribution obtained from the base classifiers. We call this the ``consensus principle''. This can be captured by the following equation:
\begin{equation}
J3=\sum_{i=i}^N ||  \F^o_{i.} - \Y^o_{i.} ||^2
\end{equation}
where $\Y^o_{ij}$ denotes that fraction of times object $O_i$ is assigned to class $j$ by the base classifiers.\\

\noindent{\bf (iv) Similarity between the group and the its average class distribution:}  The final class distribution of a group should be closer to the average class distribution of its constituent objects. This is captured by the following equation:
\begin{equation}
J4=\sum_{j=1}^G ||\F^g_{j.} - \Y^g_{j.} ||^2
\end{equation}
where $\Y^g_{ij}$ denotes that fraction of objects in group $g_i$ labeled as $j$ by the base classifiers.

We combine these four hypotheses together to formulate the following objective function parameterized by $\alpha$, $\beta$, $\gamma$ and $\delta$:

\begin{equation}\label{min2}\small
\begin{aligned}\small
{\bf P} = & \underset{\F^o, \F^g}{\text{min}}\ \  J1+J2+J3+J4=\\
& \underset{\F^o, \F^g}{\text{min}}\ \ \frac{\alpha}{2} \sum_{i=1}^N \sum_{j=1}^G \K^m_{ij}|| \F^o_{i.}-\F^g_{j.}||^2 \\
& \ \ \ \ \ \ +  \frac{\beta}{2} \sum_{i=1}^N\sum_{j=1}^N \K^c_{ij} ||\F^o_{i.} - \F^o_{j.}||^2 \\
& \ \ \ \ \ \  + \gamma \sum_{i=1}^N ||  \F^o_{i.} - \Y^o_{i.} ||^2 \\
& \ \ \ \ \ \  + \delta \sum_{j=1}^G ||\F^g_{j.} - \Y^g_{j.} ||^2\\
& \text{subject to}\ \ 
0 \leq  \beta, \gamma, \delta \leq 1,\ 0 < \alpha \leq 1,\  \frac{\alpha}{2}+ \frac{\beta}{2}+\gamma+\delta=1,\\
&\ \ \ \ \ \ \ \ \ \ \ \ \ \ \ \ \F^o_{i.}\geq 0,\  |\F^o_{i.}|=1, \forall i=1:N\\
&\ \ \ \ \ \ \ \ \ \ \ \ \ \ \ \ \F^g_{j.}\geq 0,\  |\F^g_{j.}|=1, \forall j=1:G\\
\end{aligned}
\end{equation}
Here $|.|$ and $||.||$ are 1- and 2-norm of a vector respectively. Note that $\alpha$ is not allowed to be zero (this helps in proving Theorem \ref{stric_conv}). 
Later in Section \ref{sec:parameter}, we will see that second and third components following co-occurrence and consensus principles respectively
are the most important components in the objective function, and therefore higher value of $\beta$ and $\gamma$ leads to better accuracy. $\frac{\alpha}{2}$ is used instead of $\alpha$ to simplify the proof of Theorem \ref{theory:convex} (similarly for $\beta$).

Further, each individual component can be written using the matrix form as follows:
\vspace{2mm}
\begin{equation}
\begin{split}
J1&=\alpha\{tr(\F^{o\top}\F^{o}) + tr(\F^{g\top}\F^{g} - tr(\F^{o\top} \K^m \F^{o})) \}\\
J2&=\beta. tr(\F^{o\top}\L\F^{o})\\
J3&=\gamma.||\F^o - \Y^o ||^2\\ 
 J4&=\delta.||\F^g - \Y^g ||^2
\end{split}
\end{equation}
We use these matrix forms to prove Theorem \ref{theory:convex}.

\begin{theorem}\label{theory:convex}
The optimization problem {\bf P} mentioned in (\ref{min2}) is a convex quadratic problem.
\end{theorem}

\begin{proof}
To prove that {\bf P} is convex, we have to show that both the objective function and the constraints are convex. Since the constraints in {\bf P} are all linear, they are convex. We use $f(\theta x_1 +(1-\theta)x_2) - (\theta f(x_1)+(1-\theta)f(x_2)) \leq 0, \forall \theta \in [0,1] $ to prove that the objective function is convex. 

We know that every norm is convex. Therefore, $J3=\gamma.||\F^o - \Y^o ||^2$ and 
 $J4=\delta.||\F^g - \Y^g ||^2$ are convex. 

To prove that $J2=tr(\F^{o\top}\L\F^{o})$ (ignoring the constant $\beta$) is convex, let $f=tr(\F^{o\top}\L\F^{o})$. Then
\begin{equation}\label{1}\small
\begin{split}
&f(\theta \F_1^o +(1-\theta)\F_2^o)\\
&=tr[(\theta \F_1^o + (1-\theta)F_2^o)^\top \L (\theta \F_1^o + (1-\theta)F_2^o)]\\
&=tr[\theta^2\F_1^{o\top} \L \F_1^o + (1-\theta)^2 \F_2^{o\top} \L \F_2^o\\
&+ \theta(1-\theta)(\F_1^{o\top} \L\F_2^o+\F_2^{o\top} \L \F_1^o)]
\end{split}
\end{equation}
Moreover, 
\begin{equation}\label{2}\small
\begin{split}
\theta f(\F_1^o)+(1-\theta)f(\F_2^o)=\theta tr(\F_1^{o\top} \L \F_1^o) + (1-\theta) tr(\F_2^{o\top}\L \F_2^o) .
\end{split}
\end{equation}

Substituting Eq. \ref{2} from Eq. \ref{1} yields
\vspace{2mm}
\[
\begin{split}\small
&f(\theta \F_1^o +(1-\theta)\F_2^o) - \theta f(\F_1^o)-(1-\theta)f(\F_2^o)\\
&=tr[\theta(\theta -1)(\F_1^{o\top} \L\F_1^o+\F_2^{o\top} \L\F_2^o-\F_1^{o\top} \L\F_2^o - \F_2^{o\top} \L\F_1^o)]\\
&=\theta(\theta-1)tr[(\F_1^o -\F_2^o)^\top \L(\F_1^o-\F_2^o)]
\end{split}
\]
We have to show that $\theta(\theta-1)tr[(\F_1^o-\F_2^o)^\top \L(\F_1^o-\F_2^o)]\leq 0$, or, $\theta(1-\theta)tr[(\F_1^o-\F_2^o)^\top \L(\F_1^o-\F_2^o)]\geq 0$ .

Since $\L$ is normalized graph Laplacian  matrix, it is positive semi-definite. Moreover, since both $I$ and $\W$ are symmetric, $\L$ is also symmetric. Therefore, we can write $\L=QQ^\top$ (where $Q\in \mathbf{R}^{N\times N}$). Further assume that $\Delta \F=\F_1^o-\F_2^o$. Therefore, $tr[(\F_1^o-\F_2^o)^\top \L(\F_1^o-\F_2^o)]=tr(\Delta \F QQ^\top \Delta \F)=tr[(Q\Delta\F)^\top (Q\Delta\F)]$.

{\bf Claim:} Let $\S=(Q\Delta\F)^\top (Q\Delta\F)$. Then $\S$ is symmetric and positive semi-definite.

{\bf Proof:}
By definition, $\S$ is symmetric. Moreover for any $X$, $X^\top \S X=X^\top (Q\Delta\F)^\top (Q\Delta\F) X = (Q\Delta \F X)^\top (Q\Delta \F X)=\langle(Q\Delta \F X)\cdot(Q\Delta \F X) \rangle \geq 0$ (note that  $Q\Delta \F X$ is a vector), where $\langle.\rangle$ indicates inner product of two vectors. Therefore, $\S$ is positive semi-definite.

Since $\S$ is symmetric and positive semi-definite, all its eigen vectors $\lambda_i$ are non-negative. We also know that $tr(\S)=\sum_i \lambda$. Therefore, $tr(\S)\geq 0 \Rightarrow tr[(Q\Delta\F)^\top (Q\Delta\F)=tr[(\F_1^o-\F_2^o)^\top \L (\F_1^o-\F_2^o)]\geq 0$.

Since $0\leq \theta \leq 1$,  $\theta(1-\theta)tr[(\F_1^o-\F_2^o)^\top \L(\F_1^o-\F_2^o)]\geq 0$. Therefore, $J2$ is also convex.

In the similar way, it is easy to show that each individual component of $J1=\alpha\{tr(\F^{o\top}\F^{o}) + tr(\F^{g\top}\F^{g} - tr(\F^{o\top} \K^m \F^{o})) \}$ is also convex. 

Therefore, $\mathbf{P}$ is convex.
\end{proof}

\subsection{Proposed Algorithm: \name}\label{algo}

We solve the convex quadratic optimization problem mentioned in (\ref{min2}) using standard block coordinate descent method \cite{XuY13,NIPS2009_3855}. In the $t$th iteration, if we fix $\F^{o(t)}$, the objective function boils down to the summation of the quadratic components w.r.t $\F^{g(t)}$, and it is strictly convex (see Theorem \ref{stric_conv}). Therefore, assigning $\bigtriangledown_{\F^{g(t)}_{j.}} \mathbf{P}=0$ produces the unique global minimum of the objective function w.r.t $\F^{g(t)}_{j.}$:
\begin{equation}\label{update1}\small
\F^{g(t)}_{j.}=\frac{\alpha\sum_{i=1}^N  \K_{ij}^m \F^{o(t-1)}_{i.} +2\delta \Y_{j.}^g}{\alpha \sum_{i=1}^N \K_{ij}^m + 2\delta}
\end{equation}

Similarly, if we fix $\F^{g(t)}$ 
the objective function becomes strictly convex (see Theorem \ref{stric_conv}),
and $\bigtriangledown_{\F^{o(t)}_{i.}} \mathbf{P}=0$ produces the unique global minimum w.r.t $\F^{o(t)}_{i.}$:

\begin{equation}\label{update2}\tiny
\F^{o(t)}_{i.}=\frac{\alpha \sum_{j=1}^G \K^m_{ij}\F^{g(t)}_{j.} + \beta(2\sum_{j=1}^N \K_{ij}^c \F_{j.}^{o(t)} - \K_{ii}^o \F^{o(t)}_{i.}) + 2 \gamma \Y^o_{i.}}
{\alpha \sum_{j=1}^G \K^m_{ij} + \beta ( 2 \sum_{j=1}^N \K_{ij}^c - \K_{ii}^c  ) + 2 \gamma}
\end{equation}

The pseudo-code of the proposed \name~ ({\bf E}nsemble {\bf C}lassifier by combining both {\bf C}lassification and {\bf C}lustering) algorithm is given in Algorithm \ref{alg:ec3}. In the pseudo-code, we provide the matrix form of the updates mentioned in Equations \ref{update1} and \ref{update2}.   Intuitively, in Step 5 of Algorithm \ref{alg:ec3}, the class distribution $\F^g$ of each group combines the average class distribution $\Y^g$ and the information obtained from the nodes' neighbors. Then   the updated class distribution is propagated to those neighbors by updating $\F^o$ in Step 6.

\begin{theorem}\label{stric_conv}
If we fix $\F^{o(t)}$ (resp. $\F^{g(t)}$) in {\bf P}, the resulting objective function is strictly convex w.r.t $\F^{g(t)}$ (resp. $\F^{o(t)}$).
\end{theorem}

\begin{proof}
To prove that if we fix $\F^{o(t)}$ in {\bf P} the resultant objective function is strictly convex, we need to show that $\nabla^2_{\F^{g(t)}}{\bf P}>0$.
Assume all values $\F^{o(t)}$ as constant $C$ in Equation (6) of the main text, we obtain:
\begin{equation}
\begin{split}
&\nabla_{\F^{g(t)}_{j.}}{\bf P}=-\alpha \K^m_j |C-\F^{g(t)}|+2\delta|\F^{g(t)}_{j.}-\Y^{g(t)}_{j.}|\\
&\nabla^2_{\F^{g(t)}_{j.}}{\bf P}=\alpha \K^m_j+2\delta \\
& \nabla^2_{\F^{g(t)}}{\bf P} = \sum_{j=1}^G \nabla^2_{\F^{g(t)}_{j.}}{\bf P} = \sum_{j=1}^G \alpha \K^m_j+2\delta > 0 \\
&(\text{Since $\alpha \neq 0$, and $\exists j,\K^m_j\neq 0$  })
\end{split}
\end{equation}

Therefore it is strictly convex.

Similarly to prove that if we fix $\F^{g(t)}$ in {\bf P} the resultant objective function is strictly convex, we need to show that $\nabla^2_{\F^{o(t)}}{\bf P}>0$. Assume all values $\F^{g(t)}$ as constant $D$ in Equation (6) of the main text, we obtain:
\begin{equation}
\begin{split}
\nabla_{\F^{o(t)}_{i.}}{\bf P}=&\alpha  \K^m_i |\F^{o(t)}_{i.} - D| + \beta \sum_{j=1}^N \K_{ij}^c |\F^{o(t)}_{i.} -\F^{o(t)}_{j.}| +\\ 
&\beta\sum_{k=1\wedge k\neq i}^N|\F^{o(t)}_{k.} - \F^{o(t)}_{i.}|+2\gamma |\F^{o(t)}_i-\Y^o_{i.}|\\
\end{split}
\end{equation}

\begin{equation}
\nabla^2_{\F^{o(t)}_{i.}}{\bf P}=\alpha \K^m_i+\beta (\sum_{j=1}^N \K^c_{ij}-\K^{c}_{ii})+\beta+2\gamma > 0
\end{equation}
Therefore it is also strictly convex.
\end{proof}

\if{0}
\begin{proof}
To prove that if we fix $\F^{o(t)}$ in {\bf P} the resultant objective function is strictly convex, we need to show that $\nabla^2_{\F^{g(t)}}{\bf P}>0$.
Assume all values $\F^{o(t)}$ as constant $C$ in Equation (6) of the main text, we obtain:
\begin{equation}
\begin{split}
&\nabla_{\F^{g(t)}_{j.}}{\bf P}=-\alpha \K^m_j |C-\F^{g(t)}|+2\delta|\F^{g(t)}_{j.}-\Y^{g(t)}_{j.}|\\
&\nabla^2_{\F^{g(t)}_{j.}}{\bf P}=\alpha \K^m_j+2\delta \\
& \nabla^2_{\F^{g(t)}}{\bf P} = \sum_{j=1}^G \nabla^2_{\F^{g(t)}_{j.}}{\bf P} = \sum_{j=1}^G \alpha \K^m_j+2\delta > 0 \\
&(\text{Since $\alpha \neq 0$, and $\exists j,\K^m_j\neq 0$  })
\end{split}
\end{equation}

Therefore it is strictly convex.

Similarly to prove that if we fix $\F^{g(t)}$ in {\bf P} the resultant objective function is strictly convex, we need to show that $\nabla^2_{\F^{o(t)}}{\bf P}>0$. Assume all values $\F^{g(t)}$ as constant $D$ in Equation (6) of the main text, we obtain:
\begin{equation}
\begin{split}
\nabla_{\F^{o(t)}_{i.}}{\bf P}=&\alpha  \K^m_i |\F^{o(t)}_{i.} - D| + \beta \sum_{j=1}^N \K_{ij}^c |\F^{o(t)}_{i.} -\F^{o(t)}_{j.}| +\\ 
&\beta\sum_{k=1\wedge k\neq i}^N|\F^{o(t)}_{k.} - \F^{o(t)}_{i.}|+2\gamma |\F^{o(t)}_i-\Y^o_{i.}|\\
\end{split}
\end{equation}

\begin{equation}
\nabla^2_{\F^{o(t)}_{i.}}{\bf P}=\alpha \K^m_i+\beta (\sum_{j=1}^N \K^c_{ij}-\K^{c}_{ii})+\beta+2\gamma > 0
\end{equation}
Therefore it is also strictly convex.
\end{proof}
\fi

\begin{theorem}\label{theorem:solu}
The solution obtained from Algorithm \ref{alg:ec3} satisfies the constraints  of $\mathbf{P}$ mentioned in (\ref{min2}).  
\end{theorem}
\begin{proof}
According to Step 1 of Algorithm 1 (in the main text), the initialization of both $\F^o$ and $\F^g$ should satisfy the constraints. Therefore, $\F^{o(1)}_{i.}, \F^{g(1)}_{j.}\geq 0$ and  $|\F^{o(1)}_{i.}|=1$ and  $|\F^{g(1)}_{j.}|=1, i=1:N \text{ and } j=1:G$. Moreover by definition, both the average voting matrices $\Y^o$ and $\Y^g$ satisfy the constraints, i.e., $|\Y^o_{i.}|=1$ and $|\Y^g_{i.}|=1$ . 

Let us prove the theorem by induction. Suppose, at iteration $(t-1)$ the solution satisfies the constraints, i.e., $|\F^{o(t-1)}_{i.}|=1$ and $|\F^{g(t-1)}_{j.}|=1, \forall i,j$.. From Equation 10, we obtain:
\[
\begin{split}\tiny
|\F^{g(t)}_{j.}|=\sum_{p=1}^l \F^{g(t)}_{jp}&=  \frac{\alpha  \sum_{p=1}^l\sum_{i=1}^N  \K_{ij}^m \F^{o(t-1)}_{ip} +2\delta \sum_{p=1}^l \Y_{ip}^g}{\alpha \sum_{i=1}^N \K_{ij}^m + 2\delta}\\
&=\frac{\alpha  \sum_{i=1}^N  \K_{ij}^m |\F^{o(t-1)}_{i.}| +2\delta  |\Y_{i.}^g|}{\alpha \sum_{i=1}^N \K_{ij}^m + 2\delta}=1
\end{split}
\]
Similarly, we can show that $|\F^{o(t)}_{i.}|=1$. In addition, it is clear that $\F^{g(t)}_{j.}, \F^{o(t)}_{i.}\geq 0$. Therefore, the theorem is proved.
\end{proof}

\if{0}
\begin{proof}
According to Step 1 of Algorithm 1 (in the main text), the initialization of both $\F^o$ and $\F^g$ should satisfy the constraints. Therefore, $\F^{o(1)}_{i.}, \F^{g(1)}_{j.}\geq 0$ and  $|\F^{o(1)}_{i.}|=1$ and  $|\F^{g(1)}_{j.}|=1, i=1:N \text{ and } j=1:G$. Moreover by definition, both the average voting matrices $\Y^o$ and $\Y^g$ satisfy the constraints, i.e., $|\Y^o_{i.}|=1$ and $|\Y^g_{i.}|=1$ . 

Let us prove the theorem by induction. Suppose, at iteration $(t-1)$ the solution satisfies the constraints, i.e., $|\F^{o(t-1)}_{i.}|=1$ and $|\F^{g(t-1)}_{j.}|=1, \forall i,j$.. From Equation 10, we obtain:
\[
\begin{split}\small
|\F^{g(t)}_{j.}|=\sum_{p=1}^l \F^{g(t)}_{jp}&=  \frac{\alpha  \sum_{p=1}^l\sum_{i=1}^N  \K_{ij}^m \F^{o(t-1)}_{ip} +2\delta \sum_{p=1}^l \Y_{ip}^g}{\alpha \sum_{i=1}^N \K_{ij}^m + 2\delta}\\
&=\frac{\alpha  \sum_{i=1}^N  \K_{ij}^m |\F^{o(t-1)}_{i.}| +2\delta  |\Y_{i.}^g|}{\alpha \sum_{i=1}^N \K_{ij}^m + 2\delta}=1
\end{split}
\]
Similarly, we can show that $|\F^{o(t)}_{i.}|=1$. In addition, it is clear that $\F^{g(t)}_{j.}, \F^{o(t)}_{i.}\geq 0$. Therefore, the theorem is proved.
\end{proof}
\fi

\begin{theorem}
The solution of the optimization problem {\bf P} is feasible and optimal.
\end{theorem}

\begin{proof}
Theorem \ref{theorem:solu} guarantees that the solution obtained from $\mathbf{P}$ satisfies the constraints if the initialization of both $\F^g(1)$ and $\F^o(1)$ satisfy the constraints. Moreover, we have proved in Theorem \ref{theory:convex} that $\mathbf{P}$ is convex. Therefore, any local minima is also a global minima. So the solution of the problem is both feasible and optimal.
\end{proof}
 
{\bf Handling Class Imbalance Problem:} A deeper investigation of Equations \ref{update1} and \ref{update2}
may reveal that \name~ tends to discover balanced classes. Equation \ref{update1} assigns equal weight to all the objects inside a group, and if most of the objects in the group belong to the majority class, the class distribution of the group $\F^g_{j.}$ will be biased towards the majority class. This in turn makes the class distribution of the objects $\F^o_{i.}$ biased towards the majority class in Equation \ref{update2}.  A simple solution is to perform a column-wise normalization of the objective-group membership matrix $\A^m$ as follows: $\B^m_{ij}=\frac{\A^m_{ij}}{\sum_{i=1}^N \A^m_{ij}}$, and create the bi-stochastic matrix $\K^m$ to approximate $\B^m$ \cite{NIPS2009_3855}. In the rest of the paper, we  call this version of the algorithm  \namei~ (abbreviation  of  `\name~that handles class-\textbf{i}mbalance'). We also show that \namei~performs as well as \name~in most cases or even better than \name~in some cases (See Table \ref{tab:result}). Therefore, unless otherwise mentioned, the results obtained from \namei~are reported in this paper. However, most of the characteristics of \namei~are similar to \name.

%\if{0}

{\bf Difference of \name~ from Existing Ensemble Models:} Existing supervised ensemble classifiers such as Bagging \cite{bagging}, Boosting \cite{Schapire:1999} train different base classifiers on different samples of the training set to control `bias' and `variance', whereas our method is built on a different setting where it  leverages the outputs of both supervised and unsupervised models and assigns high weight to the model which better approximates the outputs of other models. In this sense, it is also different from the traditional majority voting models. It is also different from the ensemble clustering methods such as \cite{Strehl:2003,Domeniconi:2009} because our method is essentially a classifier which requires at least one base classifier.       

%\fi

\begin{table*}[!t]
 \centering
 \caption{(a) The datasets  (ordered by the size) and their properties: number of objects, number of classes, number of features, probability of the majority class (MAJ), and entropy of the class probability distribution (ENT). Among the binary and multi-class datasets, Creditcard and Statlog are the most imbalanced respectively since MAJ value is high and ENT value is low for them. We therefore use these two datasets in Section \ref{class_imbalance} to show how the competing methods perform on imbalanced datasets. (b) Best parameter setting of \namei~for all the datasets (see Section \ref{sec:parameter}).}\label{table:dataset}
 %\vspace{-5mm}
 \scalebox{0.8}{
 \begin{tabular}{|c| l |c| c |c |c |c||cccc|}
 \hline
% \multicolumn{7}{c}{Properties of the dataset} & \multicolumn{3}{|c|}{Accuracy ($AUC$)} \\\hline 
 \multicolumn{7}{|c||}{Dataset description} & \multicolumn{4}{|c|}{Base parameter values}\\\hline
  & Dataset & \# instances & \# classes & \# features & MAJ & ENT & $\alpha$ & $\beta$ & $\gamma$ & $\delta$ \\\hline
  
\multirow{7}{0.05\columnwidth}{\rotatebox[origin=c]{90}{Binary}} 
&   Titanic \cite{titanic} & 2200 & 2 & 3 & 0.68 & 0.90 & 0.20 & 035 & 0.45 & 0 \\
 & Spambase \cite{Lichman:2013} & 4597 & 2 & 57 & 0.61 & 0.96 & 0.20 & 0.40 & 0.30 & 0.10 \\
 & Magic \cite{Lichman:2013} & 19020 & 2 & 11 & 0.64 & 0.93 & 0.15 & 0.35 & 0.40 & 0.10 \\
 & Creditcard \cite{Yeh:2009:} & 30000 & 2 & 24 & 0.78 &0.76 &0.20 & 0.40 & 0.35 & 0.05  \\
% & Wikipedia vandal & 33000 & 2 & 382 & 0.50 & 0.99 & \cite{Kumar:2015}\\													

 & Adults \cite{Lichman:2013} & 45000 & 2 & 15 & 0.75 & 0.80 & 0.25 & 0.35 & 0.35 & 0.05 \\
 
 & Diabetes \cite{Lichman:2013} & 100000 & 2 & 55 & 0.54 & 0.99 & 0.15 & 0.45 & 0.35 & 0.05  \\
 
 & Susy \cite{susy} & 5000000 & 2 & 18 & 0.52 & 0.99 &0.20 & 0.30 & 0.45 & 0.05 \\\hline				
 
 \multirow{6}{0.05\columnwidth}{\rotatebox[origin=c]{90}{Multi-class}} 
 & Iris \cite{Lichman:2013} & 150 & 3 & 4 & 0.33 & 1.58 & 0.25 & 0.35 & 0.40 & 0 \\
 
 & Image \cite{Lichman:2013} & 2310 & 7 & 19 & 0.14 & 2.78 & 0.15 & 0.45 & 0.30 & 0.10  \\
 
 & Waveform \cite{Lichman:2013} & 5000 & 3 & 21 & 0.24 & 2.48 & 0.20 & 0.35 & 0.40 & 0.05  \\
 
 & Statlog \cite{Lichman:2013} & 6435 &6 & 36 & 0.34 & 1.48 & 0.20 & 0.35 & 0.35 & 0.10 \\
 
 &Letter \cite{Lichman:2013} & 20000 & 26 & 16 & 0.04 & 4.69 & 0.25 & 0.30 & 0.40 & 0.05\\		
 &Sensor \cite{Lichman:2013} & 58509 & 11 & 49 & 0.09 & 3.45 & 0.20 & 0.35 & 0.40 & 0.05 \\\hline
 \end{tabular} }
%\vspace{-5mm}
\end{table*}

{\bf Time Complexity:} 
For each group, the time to update $\F^g_{j.}$ according to Equation \ref{update1} is $\mathcal{O}(Nl)$, totaling $\mathcal{O}(GNl)$. Similarly, according to Equation \ref{update2} updating $\F^o_{i.}$ takes $\mathcal{O}(GNl)$. Usually, a coordinate descent method takes linear time to converge \cite{Tseng2001}. Overall, the time complexity of \name~is $\mathcal{O}(GNl)=\mathcal{O}(MNl^2)$, where $M=(C_1+C_2)$, total number of base models. The complexity is therefore  linear in the number of base models and the number of objects, and quadratic in the number of classes. We will show it empirically in Section \ref{runtime}.

\begin{algorithm2e}[!t]
\DontPrintSemicolon
\SetAlgoNoEnd
\hrule
\KwIn{$\K^m$, $\K^c$, $\Y^o$, $\Y^g$, parameters: $\alpha,\beta,\gamma,\delta$, threshold: $\epsilon$}
\KwOut{$\F^o$}
\hrule 
Initialize $\F^o$ and $\F^g$ randomly such that the constraints in Equation \ref{min2} are preserved\\
Set $t=1$\\
\While{$||\F^{o(t)}- \F^{o(t-1)}||_F>\epsilon$}{
$t = t+1$\\
$\F^{g(t)}=(2\delta \mathbb{1} + \alpha\D^m)^{-1}(\alpha \K^m \F^{o(t-1)} + 2\delta \Y^g)$\label{step3}\\
$\F^{o(t)}= (\alpha \D^m+2\beta D^c- \beta I \K^c -\beta \mathbb{1}\K^c+2\gamma \mathbb{1})^{-1}(\alpha\K^m\F^{g(t)}+2\gamma \Y^o)\label{step4}  $\\
}
\Return $\F^o$;
 \hrule
\caption{\name~ Algorithm (here $\mathbb{1}$ is an all-ones matrix, $I$ is an identify matrix, $||.||_F$ is the  Frobenius norm, $D^m=diag\{(\sum_{j=1}^N \K_{ij}^m)\}|_{N\times G}$ and $D^c=diag\{(\sum_{j=1}^N \K_{ij}^c)\}|_{N\times N}$).   } \label{alg:ec3}
\end{algorithm2e}

\if 0
\begin{table}[!t]
\centering
\caption{Best parameter setting of \namei~for all the datasets.}\label{best_parameter}
\begin{tabular}{|l|cccc|}
\hline
Dataset & $\alpha$ & $\beta$ & $\gamma$ & $\delta$\\\hline
Titanic & 0.20 & 0.35 & 0.45 & 0\\
Spambase & 0.20 & 0.40 & 0.30 & 0.10\\
Magic & 0.15 & 0.35 & 0.40 & 0.10\\
Credicard & 0.20 & 0.40 & 0.35 & 0.05\\
Adult & 0.25 & 0.35 & 0.35 & 0.05\\
Diabetes & 0.15 & 0.45 & 0.35 & 0.05\\
Susy & 0.20 & 0.30 & 0.45 & 0.05\\
Iris & 0.25 & 0.35 & 0.40 & 0\\
Image & 0.15 & 0.45 & 0.30 & 0.10\\
Waveform & 0.20 & 0.35 & 0.40 & 0.05\\
Statlog & 0.20 & 0.35 & 0.35 & 0.10\\
Letter & 0.25 & 0.30 & 0.40 & 0.05\\
Sensor & 0.20 & 0.35 & 0.40 & 0.05\\\hline
\end{tabular}
\end{table}
\fi

\if{0}
\begin{table*}[!t]
 \centering
 \caption{Comparative analysis among the competing methods: (a) Rank (averaged over the ranks based on AUC and F-Score) of all the competing methods across different datasets. Top method (blue) and best baseline method   (red) are marked. (b) We also show the raw accuracy of \namei~(the best method) and the best baseline method (which varies across datasets). For few datasets UPE is not the best baseline; we explicitly report its performance below the table. (c) Runtime (in seconds) of the ensemble methods that consider both classification and clustering (we do not consider the time to run the base methods).}\label{tab:result}
 \vspace{-5mm}
 \scalebox{0.67}{
 \begin{tabular}{|c| c| c| c| c| c| c| c|| c| c| c| c||c| c|| c| c| c |c| c| c|| c| c|c|r|r|r|}
 
 \multicolumn{16}{c}{(a)} & \multicolumn{1}{c}{} & \multicolumn{5}{c}{(b)} & \multicolumn{1}{c}{} & \multicolumn{3}{c}{(c)} \\ 
 \cline{1-16}\cline{18-22}\cline{24-26}
  
% \multicolumn{7}{c}{Properties of the dataset} & \multicolumn{3}{|c|}{Accuracy ($AUC$)} \\\hline 
  & \multirow{2}{*}{Dataset} & \multicolumn{6}{c||}{Standalone Classifier} & \multicolumn{4}{c||}{Ensemble Classifier} & \multicolumn{2}{c||}{Clust. + class.} & \multicolumn{2}{c|}{Our}& & \multicolumn{3}{c||}{Best Baseline} & \multicolumn{2}{c|}{Our (\namei)} & &  \multicolumn{3}{c|}{Runtime (Seconds)} \\\cline{3-16} \cline{18-22}\cline{24-26}
  &  & DT & NB & K-NN & LR & SVM &  SGD & STA & BAG & BOO & RF & BGCM & UPE & \name & \namei &  & AUC & F-Sc & Name & AUC & F-Sc  & & BGCM & UPE & \namei \\\cline{1-16}	\cline{18-22}\cline{24-26}
  
\multirow{7}{0.05\columnwidth}{\rotatebox[origin=c]{90}{Binary}} 
&   Titanic & 8.5 & 10 & 5.5 & 8 & 7.5 & 8.5 & 14 & 8.5 & 4.5 & 8.5 & 4 & {\color{red}3.5} & 2 & {\color {blue}1} &  & 0.66  & 0.51 & UPE & 0.68 & 0.54 & & {\color {red} 27} &	31	& {\color {blue}23}
  \\ 
 & Spambase & 8 & 14 & 10.5 & 7 & 9 & 12 & 13 & 4 & 4.5 & 10.5 & 4.5 & {\color{red}3.5} & 2 & {\color {blue}1} & &0.93 & 0.92 & UPE & 0.95 & 0.94 & & {\color {red}68} & 73 & {\color {blue}67}  \\ 
 & Magic & 7 & 14 & 9 & 11 & 10 & 12 & 13 & 3.5 & 5 & 8 & 5.5 & {\color{red}3.5} & 2 & {\color {blue}1} &  & 0.55 & 0.82 & UPE & 0.58 & 0.84 & & {\color {red}510} & 621 & {\color {blue}436}  \\ 
 
 & Creditcard & 3.5 &13 & 8 & 11 & 9 & 12 & 14 & 6.5 & 5 & 10 & 6.5 & {\color{red}3} & 2 & {\color {blue}1} &  &  0.66 & 0.51 & UPE & 0.73 & 0.53 & & {\color {red}1786} & 1803 & {\color {blue}1654}\\ 
 
 &Adults & 11 & 14 & 13 & 9 & 8 & 10 & 5 & 6.5 & 3.5 & 12 & 6 & {\color{red}3} & 2 & {\color {blue}1}  &  & 0.79 & 0.69 & UPE & 0.83 & 0.71 & & {\color {red}2440} & 2519 & {\color {blue}2410}   \\\ 
 
 & Diabetes & 10 & 11 & 12 & 8.5 & 10 & 11 & 13 & 6 & 5 &7.5 &4 &{\color{red}3.5} &2 & {\color {blue}1} & & 0.65 & 0.62 &  UPE & 0.68 & 0.65 & & {\color {red}5672} & 5720 & {\color {blue}5478}   \\ 
 
 & Susy & 10.5 & 13 &13.5 &7.5 &7.5 &8 &11 &7 &{\color{red}4} &8.5 & 5.5 & 5.5 & 2 & {\color {blue}1} & & 0.77 &0.73 & BOO$^{*}$& 0.78 & 0.76  & & {\color {red}28109} & 28721 & {\color {blue}27621}\\\cline{1-16}		\cline{18-22}\cline{24-26}	
 
 \multirow{6}{0.05\columnwidth}{\rotatebox[origin=c]{90}{Multi-class}} 
 & Iris & 5 & 8.5 & 5 & 10.5 &10 & 5 & 14 & 10.5 & 9.5 &5 &4 & {\color{red}3.5} & 2.5 & {\color {blue}1} & & 0.97 & 0.96 & UPE & 0.98 & 0.98 & & {\color {red}20} & 21 & {\color {blue}14}  \\

 & Image & 7 & 14 & 7.5 &12 & 11 & 10  & 11 & 9.5 & 6.5 & 5 & 4 & {\color{red}3.5} & 2.5 & {\color {blue}1} & & 0.95 & 0.96 & UPE & 0.99 & 0.99 & & {\color {red}70} & 74 & {\color {blue}68}  \\
 
 & Waveform & 12.5 & 10.5 & 8 & 7 & 6 & 6 & 11.5 & 8.5 & 8.5 & 13.5 & 6 & {\color{red}3.5} & 2 & {\color {blue}1} & & 0.89 & 0.85 & UPE & 0.93 & 0.89 & & {\color {red}198} & 208 & {\color {blue}178} \\
  
 & Statlog & 11.5 & 12.5 & 7 & 11 & 9.5 & 7.5 & 8.5 & 5.5 & 9 & 7 & 10 & {\color{ red}3} & 2 & {\color {blue}1.5} & & 0.92 & 0.81 & UPE & 0.94 & 0.91 && {\color {red}345} & 367 & {\color {blue}248} \\
 
 &Letter & 9.5 & 8 & 10 & 9.5 & 11.5 & 9.5 & 12.5 & 5.5 & 6.5 & 5.5 & {\color{red}5} & 7.5 & 2 & {\color {blue}1}  & & 0.50 & 0.03 & BGCM$^{\diamond}$&  0.53 & 0.06 && {\color {red}1423} & 1567 & {\color {blue}1098} \\		
 &Sensor & 6.5 & 13.5 & 7.5 & 12 & 10.5 & 9 & 13 & 7.5 & 5.5 & 5.5 & {\color{red}3.5} & 4 & {\color {blue}1.5} & {\color {blue}1.5} &  &  0.98 & 0.98 & BGCM$^\#$ & 0.99  & 0.99 & &{\color {red}7992} &8092 & {\color {blue}7934} \\\cline{1-16}\cline{18-22}\cline{24-26}		
\multicolumn{2}{|c|}{Average} & 8.92 & 11.26 & 8.96 & 9.69 & 9.42& 9.19 & 11.80 & 6.65 & 5.92 & 8.15 & 5.26 & {\color{red}4.07} & 2.03 & {\color {blue}1.07} & & -- & -- & -- & 0.81 & 0.75 & & -- & -- & -- \\\cline{1-16}\cline{18-22} \cline{24-26}	  
 \multicolumn{16}{c}{} & \multicolumn{1}{c}{} & \multicolumn{9}{c}{Accuracy of UPE (AUC,F-Sc): $^{*}(0.74,0.72)$; $^{\diamond}(0.49,0.03)$;$^{\#}(0.97,0.97)$} \\ 

 \end{tabular} }
 \vspace{-5mm}
\end{table*}
\fi

\section{Experimental Setup}\label{setup}
In this section, we briefly explain the experimental setup -- datasets used in our experiments, set of base classifiers and base clustering methods whose outputs are combined, and set of baseline methods with which we compare our method.

 \textbf{{Datasets:}}
We perform our experiments on a collection of 13 datasets, most of which are taken from the standard UCI machine learning repository  \cite{Lichman:2013}. These datasets are used widely and highly diverse in nature in terms of the size, number of features and the distribution of objects in different classes.
%These datasets are highly diverse (in terms of size, class distribution, feature size) and widely used. 
A summary of these datasets is shown in Table \ref{table:dataset}. In each iteration, we randomly divide each dataset into three segments -- 60\%  for training, 20\% for parameter selection (validation), and  20\%  for testing. 
We use this division to train our base classifiers. However, base clustering methods are run on the entire dataset (combining training, validation and testing). 
The outputs of the base classifiers and base clustering methods on only the test dataset are fed into our method.  
We report the accuracy in terms of AUC (Area under the ROC curve) and F-Score for each dataset by averaging the results over $20$ such iterations. 
The predictive results of base methods on the test set are provided to our methods.
%We use standard grid search to figure out the appropriate values of $\alpha$, $\beta$, $\gamma$ and $\delta$. \\

 \textbf{ {Base Classifiers:}}
In this study, we use seven (standalone) base classifiers: (i) DT: CART algorithm for decision tree with Gini coefficient \cite{Quinlan:1986}, (ii) NB: Naive Bayes algorithm with kernel density estimator \cite{Webb2010}, (iii) K-NN: K-nearest neighbor algorithm \cite{citeulike:5847607}, (iv)  LR: multinomial logistic regression \cite{Krishnapuram}, (v)  SVM: Support Vector Machine with linear kernel \cite{Steinwart:2008}, (vi) SGD:  stochastic gradient descent classifier \cite{Bottou2010} and (vii) Convolutional Neural Networks (CNN)\footnote{https://github.com/fastai/courses} \cite{Krizhevsky:2012}.
We utilize standard grid search for hyper-parameter optimization. These algorithms are further used later as standalone baseline classifiers to compare with our ensemble methods.

 \textbf{ {Base Clustering Methods:}} We consider five state-of-the-art clustering methods: DBSCAN \cite{ester1996}, Hierarchical (with complete linkage and Euclidean distance) \cite{Rokach2005}, Affinity \cite{Frey07clusteringby}, K-Means \cite{kmeans} and MeanShift \cite{Fukunaga:2006}.  The value of $K$ in K-Means clustering is determined by  the Silhouette Method \cite{Rousseeuw:1987}. Other parameters of the methods are systematically tuned to get the best performance.

  \textbf{ {Baseline Classifiers:}}
We compare our methods with $7$ standalone classifiers mentioned earlier. We additionally compare them with $5$ state-of-the-art ensemble classifiers: (i) Linear Stacking (STA):  stacking with multi-response linear regression \cite{Reid2009}, 
%(MLR)\footnote{We observed that stacking with MLR outperforms stacking with meta decision tree and stacking with multi-response model trees.} \cite{}, 
(ii) Bagging (BAG): bootstrap aggregation method \cite{bagging}, (iii) AdaBoost (BOO): Adaptive Boosting \cite{Schapire:1999}, (iv) XGBoost (XGB): a tree boosting method \cite{Chen:2016:XST}, and (v)  Random Forest (RF): random forest  with Gini coefficient \cite{Breiman:2001}. Moreover, we compare our methods with both BGCM \cite{NIPS2009_3855} and UPE \cite{Ao:2014}, two recently proposed  consensus maximization approaches that combine both classifiers and clustering methods. 
Thus, in all, we compare our method with 14 classifiers including sophisticated ensembles.

%\noindent\textbf{\underline{Baseline Clustering Methods:}} Although 

\begin{table*}[!ht]
 \centering
 \caption{Accuracy of the competing methods in terms of (a) AUC and (b) F-Score. Top three results per dataset are in boldface. We observe that there is  no  particular  baseline  which  is  the  best  across  all  datasets. UPE, BGCM and BOO stand as best baselines depending upon the datasets. However,  \namei~is a single algorithm that achieves the best performance irrespective of any dataset. }\label{tab:result}
\vspace{-5mm}
 \scalebox{0.85}{
 \begin{tabular}{|c| c| c| c| c| c| c|c| c|| c|c| c| c| c||c| c|| c| c| }

 \multicolumn{18}{c}{(a)}   \\\cline{1-18}
  
% \multicolumn{7}{c}{Properties of the dataset} & \multicolumn{3}{|c|}{Accuracy ($AUC$)} \\\hline 
  & \multirow{2}{*}{Dataset} & \multicolumn{7}{c||}{Standalone Classifier} & \multicolumn{5}{c||}{Ensemble Classifier} & \multicolumn{2}{c||}{Clust. + class.} & \multicolumn{2}{c|}{Our} \\\cline{3-18}

 &  & DT & NB & K-NN & LR & SVM &  SGD & CNN & STA & BAG & BOO & XGB  &RF  & BGCM & UPE & \name & \namei   \\\cline{1-18}	
  
\multirow{7}{0.05\columnwidth}{\rotatebox[origin=c]{90}{Binary}} 
&   Titanic & 0.655 & 0.659 & 0.667 & 0.664 & 0.664 & 0.664 & 0.674&0.500 & 0.655 & 0.664 & 0.659&0.655 & 0.664 & {\bf 0.665}  & {\bf 0.677} & {\bf 0.687}   \\ 
 & Spambase & 0.909 & 0.850 & 0.872 & 0.914 & 0.903 & 0.867 & 0.916&0.864 & 0.931  &0.933&0.930 & 0.897 & 0.931 & {\bf 0.937} & {\bf 0.952} & {\bf 0.954}  \\ 
 & Magic & 0.519 & 0.420 & 0.503 & 0.477 & 0.479 & 0.468 & 0.470&0.440 & 0.555 & 0.553 & 0.543&0.518 & 0.553 & {\bf 0.554} & {\bf 0.560} & {\bf 0.587} \\ 
 & Creditcard & 0.644 & 0.612 & 0.635 & 0.621 & 0.626 & 0.621 &0.623 &0.611 & 0.661 & 0.641 & 0.642&0.631 & 0.656 & {\bf 0.666} & {\bf 0.702} & {\bf 0.732} \\ 
 
 &Adults & 0.743 & 0.78 & 0.737 & 0.766 & 0.767 & 0.751 &0.772 &0.785 & 0.778 & 0.793 & 0.783&0.742 & 0.786 & {\bf 0.793} & {\bf 0.804} & {\bf 0.834}  \\\ 
 
 & Diabetes & 0.573 & 0.505 & 0.566 & 0.614 & 0.614 & 0.612 & 0.603&0.572 & 0.643  & 0.648&0.634 & 0.574 & 0.643 & {\bf 0.653} & {\bf 0.676} & {\bf 0.687}   \\ 
 
 & Susy & 0.690 & 0.699 & 0.664 &0.721 & 0.762 & 0.734 &0.741 &0.731 & 0.766 & {\bf 0.772} & 0.770&0.701 & 0.759 & 0.746 & {\bf 0.774} & {\bf 0.786} \\\cline{1-18}		
 
 \multirow{6}{0.05\columnwidth}{\rotatebox[origin=c]{90}{Multi-class}} 
 & Iris & 0.950 & 0.950 & 0.950 & 0.925 & 0.925 & 0.932 &0.941 & 0.675 & 0.925 & 0.912 & 0.910&0.95 & 0.932 & {\bf 0.975} & {\bf 0.986} & {\bf 0.989}  \\

 & Image & 0.929 & 0.873 & 0.921 & 0.948 & 0.912 & 0.904 &0.901 & 0.903 & 0.906 & 0.918 & 0.920&0.912 & 0.931 & {\bf 0.951} & {\bf 0.991} & {\bf 0.994} \\
 
 & Waveform & 0.831 & 0.858 & 0.864 & 0.903 & 0.903 & 0.898 & 0.902&0.847 & 0.897 & 0.903 & 0.892&0.828 & 0.892 & {\bf 0.903} & {\bf 0.921} & {\bf 0.931} \\
 
 & Statlog & 0.897 & 0.879 & 0.917 & 0.892 & 
0.886 & 0.901 & 0.901&0.898 & 0.910 & 0.914& 0.904 & 0.896 & 
0.821 & {\bf 0.921} & {\bf 0.958} & {\bf 0.943}\\
 
 &Letter &  0.499 & 0.500 & 0.500 & 0.499 & 0.499 & 
0.499 &0.499 &0.499 & 0.500 & 0.501 & 0.501&0.500 & 
{\bf 0.502} & 0.491 & {\bf 0.531} & {\bf 0.531}
 \\		
 &Sensor & 0.980 & 0.846 & 0.975 & 0.862 & 0.846 & 0.915 & 0.972&0.753 & 0.977 & 0.934 &0.971 &0.971 & {\bf 0.980} & 0.971 & {\bf 0.995} & {\bf 0.996}  \\\cline{1-18}		
 
\multicolumn{2}{|c|}{Average} & 0.754 & 0.730 & 0.753 & 0.757 & 0.751 & 0.751 & 0.762&0.698   &0.778 & 0.777 &0.773 &0.752  &0.774 & {\bf 0.786} & {\bf 0.808} & {\bf 0.819}  \\\hline

\multicolumn{18}{c}{}   \\

\multicolumn{18}{c}{(b)}   \\\cline{1-18}

 & \multirow{2}{*}{Dataset} & \multicolumn{7}{c||}{Standalone Classifier} & \multicolumn{5}{c||}{Ensemble Classifier} & \multicolumn{2}{c||}{Clust. + class.} & \multicolumn{2}{c|}{Our} \\\cline{3-18} 
 &  & DT & NB & K-NN & LR & SVM &  SGD & CNN & STA & BAG & BOO & XGB & RF & BGCM & UPE & \name & \namei   \\\cline{1-18}	
  
\multirow{7}{0.05\columnwidth}{\rotatebox[origin=c]{90}{Binary}} 
&   Titanic & 0.476 & 0.461 & 0.461  & 0.421 & 0.446 & 
0.416 & 0.476&0.054 & 0.476 & 0.491 & 0.483&0.476 & 0.501 & {\bf 0.513} & 
{\bf 0.528} & {\bf 0.541}\\ 

 & Spambase & 0.891 & 0.813 & 0.879 & 0.898 & 0.885 & 0.844 &0.882 &0.840 & 0.913 & 0.911 & 0.901&0.877 & 0.912 & {\bf 0.923} & 
{\bf 0.941} & {\bf 0.944}  \\ 

 & Magic & 0.741 & 0.495 & 0.710 & 0.653 & 0.659 & 
0.632 & 0.751&0.543 & 0.820 & 0.817 &0.812 &0.740 & 0.812 & {\bf 0.821} &  
{\bf 0.830} & {\bf 0.842} \\ 

 & Creditcard & 0.499 & 0.399 & 0.481 & 0.372 & 0.432 &
0.395 & 0.446&0.371 & 0.491 & 0.495 & 0.489& 0.426 & 0.460 & {\bf0.509} &  
{\bf 0.529} & {\bf 0.531}  \\ 
 
 &Adults & 0.614 & 0.553& 0.611 & 0.665 & 0.665 & 0.639 &0.652 &0.689 & 0.683 & 0.689 & 0.681&0.612 & 0.689 & {\bf 0.691} & {\bf0.701} & {\bf0.716} \\\ 
 
 & Diabetes & 0.540 & 0.567 & 0.526 & 0.529 & 0.525 & 0.525 & 0.515&0.337 & 0.601 & 0.605 &0.609 &0.610 & 0.611  & {\bf 0.621} &
{\bf0.641} & {\bf0.651 }
    \\ 
 
 & Susy & 0.672 & 0.537 & 0.606 & 0.667 & 0.672 & 0.676 & 0.671&
0.650 & 0.672 & {\bf0.731} & 0.721&0.682 & 0.724 & 0.721 & {\bf0.756} & {\bf0.763} 
  \\\cline{1-18}	
  
 \multirow{6}{0.05\columnwidth}{\rotatebox[origin=c]{90}{Multi-class}} 
 
 & Iris & 0.932 & 0.667 & 0.932 & 0.897 & 0.899 & 0.883 & 0.910&
0.535 & 0.897 & 0.932 & 0.912&0.932 & 0.946 & {\bf0.961} & {\bf0.987 }& {\bf0.988}\\

 & Image & 0.965 & 0.763 & 0.963 & 0.914 &  0.919 & 
0.921 &0.955 &0.926 & 0.909 & 0.962 &0.951 &0.960 & 0.960 & {\bf 0.967} &  
{\bf0.978} & {\bf0.991} 
  \\
 
 & Waveform & 0.774 & 0.800 & 0.818 & 0.812 & 0.813 &
0.831 &0.791 &0.792 & 0.811 & 0.761&0.772 &0.770 & 0.850 & {\bf0.856}& 
{\bf0.884} & {\bf0.896}\\
 
 & Statlog & 0.661 & 0.671 & 0.718 & 0.681 & 0.789  & 0.781 &0.704 &0.771 & 0.802 & 0.664 &0.792 &0.802  & 0.801 & {\bf 0.810}  & 
{\bf 0.891}  & {\bf 0.912} \\
 
 &Letter & 0.030 & 0.032 & 0.030  & 0.030 & 0.032  & 0.030 & 0.031&0.006 & 0.038  &0.028& 0.031& 0.030 & {\bf 0.033} &  0.031 & {\bf 0.0505}  & {\bf 0.067} 
   \\		
 &Sensor & 0.964 &0.647 &0.955 &0.748 &0.846 &0.915 &0.942 &0.753 &0.761 &0.771& 0.762&0.965 &{\bf 0.981} &0.970 & {\bf 0.997} & {\bf 0.993} 
  \\\cline{1-18}		
\multicolumn{2}{|c|}{Average} & 0.669&0.577&0.666&0.638&0.656&0.662&0.673 &0.559&0.685&0.683& 0.681&0.686&0.717&{\bf 0.723}&{\bf 0.747}&{\bf 0.757}
 \\\hline
 \end{tabular} }
 %\vspace{-5mm}
\end{table*}

\section{Experimental Results}\label{result}
In this section, we present experimental results in details. We start by defining our parameter selection method, followed by comparative analysis with the baselines. We then present a detailed understanding of our method -- i.e.,  how it depends on the base methods, how it handles imbalanced data, how robust it is to random noise injected into the base solutions, how each of the components in the objective function affects the performance of the model, and how  its runtime depends on various parameters of the datasets (such as number of objects, classes and base methods). 

\subsection{Parameter Selection}\label{sec:parameter}
Our proposed methods depend on the values of $\alpha, \beta, \gamma, \delta$. 
Therefore, appropriate parameter selection would lead to better accuracy. 
%Since our method does not access raw data directly, we cannot use existing parameter selection techniques (e.g., grid search) often used in supervised/semi-supervised methods. 
Here, we conduct an exhaustive experiment to understand the appropriate values of the parameters used in our methods as follows. For each dataset, we vary  the value of each parameter between $0$ and $1$ with an increment of $0.05$. We then choose only those values of the parameters for which the accuracy of our methods in terms of AUC falls in the top 10 percentile of the entire accuracy range. Figure \ref{fig:parameter} shows the fraction of selected values for parameters of \namei~falling in certain ranges for all the datasets\footnote{The pattern is same for \name.}. We observe that $\beta$ and $\gamma$ always get  higher values, followed by $\alpha$ and $\delta$. We therefore conclude that the components that follow both co-occurrence and consensus principles (mentioned in Section \ref{method}) are the most effective components of our objective function (Equation \ref{min2}). However, the other two parameters $\alpha$ and $\delta$ are also important. Therefore, we suggest the following ranges for the parameters: $0.10\leq \alpha \leq 0.40$, $0.30 \leq \beta, \gamma \leq 0.60$,  and $0 \leq \delta \leq 0.20$. In the rest of the paper, we report the results with the following parameter setting for both \name~and \namei: $\alpha=0.25$, $\beta=0.35$, $\gamma=0.35$ and $\delta=0.05$ (See Table \ref{table:dataset} for the best parameter setting of \namei~for individual datasets).

%\vspace{-3mm}
\begin{figure}[!t]
 \centering
 \includegraphics[width=\columnwidth]{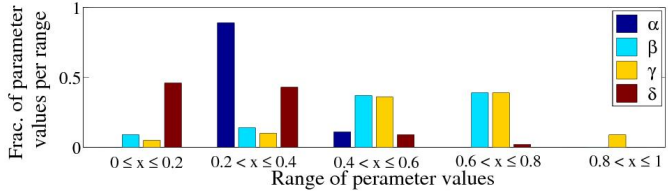}
 \vspace{-5mm}
 \caption{(Color online) Fraction of values of each parameter within a certain range across all datasets. For each dataset, we consider only those parameter values for which AUC of \namei~lies in the top 10 percentile of the entire accuracy range. 
 %The pattern is same for \name.
 }\label{fig:parameter}
% \vspace{-5mm}
 \end{figure}

\begin{figure}[!h]
\vspace{-2mm}
\centering
 \includegraphics[width=\columnwidth]{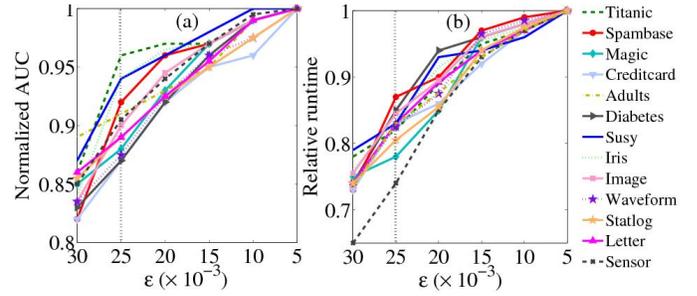}
 \vspace{-7mm}
 \caption{(Color online) (a) Normalized AUC, and (b) relative runtime of \namei~with the decrease of $\epsilon$. }\label{fig:epsilon}
%\vspace{-5mm}
 \end{figure}

Another parameter, $\epsilon$ controls the convergence of \name~ -- the higher the value of $\epsilon$, the faster the convergence of \name; however we may sacrifice the performance. To understand the trade-off between performance and runtime, we decrease the value of $\epsilon$ from $0.030$ to $0.005$ (with the decrement of $0.005$) and measure the accuracy and runtime. Figure \ref{fig:epsilon} shows that on average, considering $\epsilon=0.025$, \namei~can obtain 90\% of the maximum accuracy (with $\epsilon=0.005$) and $82\%$ of the maximum runtime (with $\epsilon=0.005$); whereas with  $\epsilon=0.020$ ({\em resp.}  $\epsilon=0.030$) the average accuracy would be $93\%$ ({\em resp.} $0.84\%$), and the average runtime would be $89\%$ ({\em resp.} $74\%$) of the maximum accuracy and runtime respectively. Therefore, rest of the results are reported with $\epsilon=0.025$.

\subsection{Comparison with Baseline Classifiers}\label{compare}
We evaluate the performance of the competing methods using two metrics -- AUC and F-Score. The values of both the metrics range between $0$ and $1$; the higher the value, the higher the accuracy. 
%For better representation, we rank the methods based on each of these metrics separately, and compute an average rank for each method. Therefore, the average rank of $1$ for a method indicates that  the method is best. 
Table \ref{tab:result} shows the accuracy of all the methods for different datasets in terms of AUC and F-Score. Overall, aggregating clustering and classification (e.g., BGCM, UPA, \name, \namei) always provides better accuracy compared to aggregating only classifiers (e.g., STA, BAG) or standalone classifiers (e.g., DT, SVM).
We observe that our proposed methods (\name~ and \namei)  always outperform others for all the datasets. 
%We also report the raw accuracy of \namei~and the best baseline method (which varies across datasets) in Table \ref{tab:result}(b).
In most cases, UPE turns out to be the best baseline, followed by BGCM. However, irrespective of the datasets, \namei~acheives average AUC of $0.82$ ({\em resp.} average F-Score of $0.76$), which is  4.2\% ({\em resp.} 4.7\%) higher than UPE. \namei~gains maximum improvement over UPE for the Creditcard dataset (10\% in terms of AUC), which is significant according to the $t$-test with 95\% confidence interval. Moreover, as the network size increases. the improvement of both \name~and \namei~compared to the best baseline also increases. However, both UPE and BGCM seem to be very competitive with an average AUC of $0.78$ and $0.77$ respectively. Further, we observe in Table \ref{tab:result} that 
there is  no  single  baseline method  which  is  the  best  across  all  datasets -- UPE, BGCM and BOO stand as best baselines depending upon the datasets. However,  \namei~is a single algorithm that achieves the best performance across all the datasets. {\em One may therefore choose \namei~as opposed to investing time settling on which classifier to choose in light of the fact that \namei~ is  on a par with any existing classifier irrespective of the datasets used.}

%\vspace{-3mm}

\subsection{Effect of Base Classifiers}\label{sec:base_class}
A crucial part of our methods is to select the appropriate base classifiers. Here we seek to answer the following question -- how is our method affected by the quality and the number of base classifiers?

{\bf Quality of Base Classifiers:} To understand which base classifier has the highest impact, we drop each base classifier in isolation and measure the performance of \namei. Table \ref{impact_classifier} shows the performance of \namei~on different datasets. For Creditcard, we observe maximum deterioration (12.3\% and 13.20\% drop in terms of AUC and F-Score respectively) when Decision Tree (DT)  is dropped, which is followed by  K-NN, CNN, LR, SVM, SDG and NB. Interestingly, this rank of base classifiers is highly correlated with the rank obtained based on their individual performance on the Creditcard dataset as shown in Table \ref{tab:result} (the rank is DT, K-NN, CNN, SVM, LR, SGD and NB). A similar pattern is observed for the Waveform dataset, where dropping SVM has the highest effect  on the accuracy of \namei~(7.52\% and  14.60\% drop in terms of AUC and F-Score respectively), followed by CNN, SGD, LR, DT, K-NN and NB; and this rank is highly correlated with their individual performance as well. This pattern is remarkably similar for the other datasets (see Table \ref{impact_classifier}).

{\bf Number of Base Classifiers:}
Further, to understand the optimal number of base classifiers that need to be added into the base set, we add each  classifier one at a time into the base set based on the impact of its quality on \namei~as reported in Table \ref{impact_classifier}. For example, for Creditcard, we add the classifiers by the following sequence -- DT, K-NN, CNN, LR, SVM, SDG, NB, and measure the accuracy of \namei. Table \ref{impact_classifier1} shows that the rate of increase of \namei's accuracy is quite significant ($t$-test with 95\% confidence interval) till the addition of $4$ classifiers out of $6$. However, strong classifies seem to be more useful to enhance the accuracy. 

From both these observations,  we conclude that {\em while selecting base methods, one should first consider strong standalone classifiers}. However, addition of a weak classifier to the base set never deteriorates performance as long as a sufficient number of strong classifiers are present for aggregation.

\begin{table*}[!t]
\centering
\caption{Impact of each base classifier on the accuracy of \namei. We drop each base classifier in isolation and measure AUC of \namei~on all the datasets. }\label{impact_classifier}
\scalebox{0.9}{
\begin{tabular}{ l l cc cccccc ccccc}
\hline
\multirow{2}{*}{No.} & Base  & Titanic & Spambase & Magic & Creditcard & Adults & Diabetes & Susy & Iris & Image & waveform & Statlog & Letter & Sensor \\
 & classifier &   &   &   &  \\\hline
(i) & All       &0.68&	0.95&	0.58&	0.73&	0.83&	0.68&	0.78&	0.98&	0.99&	0.93&	0.94 &	0.53&	0.99\\\hline
(ii) & (i) - DT & 0.64	&0.91&	0.49&	0.64&	0.8&	0.63&	0.75&	0.90&	0.92&	0.90 &	0.92&	0.53&	0.90  \\ 
(iii) & (i) - NB & 0.67&	0.94&	0.58	&0.73&	0.81&	0.67&	0.76&	0.94&	0.98&	0.91&	0.93&	0.43&	0.98\\
(iv) & (i) - K-NN & 0.6 &	0.93&	0.52	&0.66	&	0.81	&	0.66&	0.77	&0.92	&0.93	&0.91&	0.82&	0.49	&0.91
 \\
(v) & (i) - LR & 0.62&	0.88	&0.55&	0.69	&0.78	&0.6&	0.66&	0.95&	0.97&	0.89	&0.89&	0.47&	0.96\\
(vi) & (i) - SVM & 0.63&	0.93&	0.53&	0.7&	0.75&	0.61&	0.69&	0.97	&0.97	&0.86	&0.87&	0.52	&0.93
\\
(vii) & (i) - SGD & 0.65&	0.94&	0.57&	0.7	&0.83&	0.63&	0.71&	0.91&	0.95&	0.87&	0.86&	0.5&	0.93\\
(vii) & (i)-CNN & 0.63 & 0.92 & 0.50 & 0.67 &0.82 & 0.62  & 0.67 & 0.91 & 0.93 & 0.87 &  0.93 & 0.54  & 0.91
\\
\hline
\end{tabular}}

\end{table*}

\if{0}
\begin{table}[t]
\centering
\caption{Impact of each base classifier on the accuracy of \namei. We drop each base classifier in isolation and measure the accuracy of \namei~on two representative datasets: Creditcard and Waveform. 
%Results are identical for other datasets. 
Lowest accuracy is marked in boldface.}\label{impact_classifier}
\vspace{-3mm}
\begin{tabular}{ l l cc cc }
\toprule
\multirow{2}{*}{No.} & Base  & \multicolumn{2}{c}{Creditcard} & \multicolumn{2}{c}{Waveform}\\\cline{3-6}
 & classifier & AUC & F-Sc & AUC & F-Sc\\\hline
(i) & All       &  0.73 & 0.53            & 0.93 & 0.89\\\hline
(ii) & (i) - DT & {\bf 0.64} & {\bf 0.46} & 0.90 & 0.83  \\ 
(iii) & (i) - NB & 0.73 & 0.54 & 0.91 & 0.88\\
(iv) & (i) - K-NN & 0.66 & 0.49 & 0.91 & 0.86 \\
(v) & (i) - LR & 0.69 & 0.50 & 0.89 & 0.81\\
(vi) & (i) - SVM & 0.70 & 0.51 & {\bf 0.86} & {\bf 0.76}\\
(vii) & (i) - SGD & 0.70 & 0.52 & 0.88 & 0.80\\
(viii) & (i) - CNN & 0.67 & 0.51 & 0.87 & 0.77\\
\bottomrule
\end{tabular}
%\vspace{-2mm}
\end{table}
\fi

\begin{table}[!t]
\centering
\caption{Impact of the number of base classifiers  on the performance of \namei. We add each classifier one at a time (based on the decreasing order of the impact on \namei~as reported in Table \ref{impact_classifier}) and measure the accuracy  of \namei~on two representative datasets: Creditcard and Waveform. 
%Results are identical for other datasets. 
} \label{impact_classifier1}
%\vspace{-3mm}
\scalebox{0.9}{
\begin{tabular}{ l l cc | ll cc }
\hline
\multirow{2}{*}{No.} & Base  & \multicolumn{2}{c|}{Creditcard} & \multirow{2}{*}{No.} & Base &  \multicolumn{2}{c}{Waveform}\\\cline{3-4} \cline{7-8}
 & Classifier & AUC & F-Sc & & Classifier & AUC & F-Sc\\\toprule
(i) & \namei+DT & 0.58 & 0.39 & (i) & \namei+SVM & 0.72& 0.68  \\
(ii) & (i)+K-NN & 0.64 & 0.46 & (ii) & (i)+CNN & 0.77 & 0.72 \\
(iii) & (ii)+CNN &  0.66& 0.47 &  (iii) & (ii)+SGD & 0.81&0.74 \\
(iv) & (iii)+LR & 0.68& 0.49& (iv) & (iii)+LR & 0.86&0.81 \\
(v) & (iv)+SVM & 0.71 & 0.52& (v) & (iv)+DT & 0.90 &0.85 \\
(vi) & (v)+SDG & 0.72 & 0.53& (vi) & (v)+K-NN & 0.91& 0.88\\
(vii) & (vi)+NB & 0.73 & 0.53& (vii) & (vi)+NB & 0.93 & 0.89 \\\bottomrule
\end{tabular}}
%\vspace{-5mm}
\end{table}

\subsection{Effect of Base Clustering Methods}\label{sec:base_cluster}
We are also interested to see the effect of base clustering methods on the performance of our method. We start by measuring the performance of individual base clustering methods. Since clustering does not provide actual class information, we consider this  as an unsupervised learning problem and group the objects in the test set based on the ground-truth class information.  We then check how well a clustering method captures the ground-truth based groups. The accuracy is reported in terms of Normalized Mutual Information  (NMI) \cite{Paninski:2003}. Table \ref{cluster_accuracy} shows that on average Affinity Clustering outperforms others, followed by Mean-Shift and DBSCAN. 
%In general, the performance of the clustering methods is poor compared to classification,  which is justifiable because it is difficult for any clustering method to detect small number of large-size ground-truth clusters. 

\begin{table}[!h]
\centering
\vspace{-3mm}
\caption{Accuracy of base clustering methods in terms of NMI for different datasets. The accuracy of top (second) ranked method is marked in blue (red).}\label{cluster_accuracy}
\vspace{-3mm}
\scalebox{0.85}{
\begin{tabular}{l| c c c c c}
\toprule
Dataset & DBSCAN & Hierarchical & Affinity & K-Means & Mean-Shift \\\hline
Titanic  & 0.38 & 0.29 & {\color{red}0.40} & 0.34 & {\color{blue}0.43} \\
Spambase & {\color{red}0.30} & 0.23 & {\color{blue}0.35} & 0.30 & 0.29\\
Magic    & {\color{red}0.31} & 0.24 & {\color{blue}0.33} & 0.21 & 0.28\\
Credicard & 0.36 & 0.22 & {\color{blue}0.41} & 0.29 & {\color{red}0.39}\\
Adults & 0.42 & 0.31 & {\color{blue}0.45} & 0.28 & {\color{red}0.33}\\
Diabetes & 0.30 & 0.25 & {\color{blue}0.36} & 0.24 & {\color{red}0.37}\\
Susy & 0.29 & 0.27 & {\color{blue}0.39} & 0.31 & {\color{red}0.33}\\\hline
Iris & 0.41 & 0.36 & {\color{red}0.47} & 0.24 & {\color{blue}0.48}\\
Image & 0.44 & 0.28 & {\color{blue}0.48} & 0.33 & {\color{red}0.45}\\
Waveform & 0.33 & 0.29 & {\color{blue}0.39} & 0.31 & {\color{red}0.37}\\
Statlog & {\color{red}0.49} & 0.31 & {\color{blue}0.51} & 0.35 & 0.45\\
Letter & 0.43 & 0.34 & {\color{blue}0.44} & 0.34 & {\color{red}0.38}\\
Sensor & 0.49 & 0.31 & {\color{red}0.50} & 0.28 & {\color{blue}0.53}\\\hline
Average & 0.38 & 0.29 & {\color{blue}0.42}  & 0.28 & {\color{red}0.39}\\
\bottomrule
\end{tabular}}
%\vspace{-3mm}
\end{table}

\begin{table}[!h]
\centering
%\vspace{-2mm}
\caption{Impact of the quality of base clustering methods on the performance of \namei. We drop each base clustering method in isolation and measure the accuracy  of \namei~on two representative datasets: Creditcard and Waveform. Results are identical for other datasets. Lowest accuracy is marked in boldface.} \label{impact_clustering}
%\vspace{-3mm}
\begin{tabular}{ l l cc cc }
\toprule
\multirow{2}{*}{No.} & Base  & \multicolumn{2}{c}{Creditcard} & \multicolumn{2}{c}{Waveform}\\\cline{3-6}
 & Clustering & AUC & F-Sc & AUC & F-Sc\\\hline
(i) & All       &  0.73 & 0.53            & 0.93 & 0.89\\\hline
(ii) & (i) - DBSCAN & 0.71 & 0.51 & 0.91 & 0.87 \\ 
(iii) & (i) - Hierarchical & 0.73 & 0.52 & 0.92 & 0.86 \\
(iv) & (i) - Affinity & {\bf 0.69} & {\bf 0.48} & {\bf 0.87} & {\bf 0.83} \\
(v) & (i) - K-Means & 0.72 & 0.53 & 0.92 & 0.88\\
(vi) & (i) - Mean-Shift & 0.71 & 0.50 & 0.90 & 0.85\\
\bottomrule
\end{tabular}
\vspace{-5mm}
\end{table}

{\bf Quality of Base Clustering Methods:}
To show how the quality of each base clustering method affects the performance of \namei, we perform a similar experiment to the one mentioned in Section \ref{sec:base_class} -- we drop each base clustering method in isolation and  report the accuracy of \namei~in Table \ref{impact_clustering}. Once again similar pattern is noticed -- Affinity Clustering which seems to be the best standalone clustering method for Credicard and Waveform (as shown in Table \ref{cluster_accuracy}), turns out to be the best base clustering method whose deletion leads to higher decrease of \namei's performance.  Interestingly, the decrease in performance of \namei~due to dropping the best base classifier (DT) is higher than the same due to the best base clustering method (Affinity) (see Tables \ref{impact_classifier} and \ref{impact_clustering} for comparison). This may indicate that {\em the effect of classifiers in our method is higher than that of a clustering method} -- this may be justifiable due to the fact that a strong standalone  base classifier itself is capable of producing significantly accurate result, and  we essentially leverage the solution of base classifiers to produce the final prediction.   

{\bf Number of Base Clustering Methods:} In order to understand how our method is affected by the number of base clustering methods, we run \namei~with one clustering method added at a time (based on the decreasing impact on \namei~as shown in Table \ref{impact_clustering}) in the base set and report the accuracy in Table \ref{impact_clustering1}. We observe that -- as oppose to the case for base classifiers (shown in Table \ref{impact_classifier1}), \namei~with only Affinity  and Mean-Shift clustering methods achieves almost 95\% of the accuracy obtained when all 5 base clustering methods are present. This result corroborates the conclusion drawn in \cite{NIPS2009_3855} that {\em a small number of strong base clustering methods might be enough to obtain a near optimal result}. However, once again, the performance of \namei~never deteriorates with the addition of new clustering methods into the set of base methods.

\begin{table}[!t]
\centering
\caption{Impact of the number of base clustering methods  on the performance of \namei. We add each clustering method one at a time (based on the decreasing order of the impact on \namei~ as reported in Table \ref{impact_clustering}) and measure the accuracy of \namei~on two representative datasets: Creditcard and Waveform. Results are identical for other datasets. } \label{impact_clustering1}
\vspace{-2mm}
\begin{tabular}{ l l cc cc }
\hline
\multirow{2}{*}{No.} & Base  & \multicolumn{2}{c}{Creditcard} & \multicolumn{2}{c}{Waveform}\\\cline{3-6}
 & Clustering & AUC & F-Sc & AUC & F-Sc\\\hline
(i) & \namei+Affinity & 0.66 & 0.47 & 0.84 & 0.80  \\
(ii) & (i)+Mean-Shift & 0.70 & 0.50 & 0.89 & 0.84 \\
(iii) & (ii)+DBSCAN & 0.71 & 0.50 & 0.91 & 0.86\\
(iv) & (iii)+K-Means & 0.73 & 0.52 & 0.92 & 0.87 \\
(v) & (iv)+Hierarchical &  0.73 & 0.53            & 0.93 & 0.89\\
\hline
\end{tabular}
%\vspace{-5mm}
\end{table}

%\vspace{-3mm}

\subsection{Importance of Individual Components of the Objective Function}
Our proposed objective function mentioned in Equation \ref{min2} is composed of four components. One might wonder how important each of these components are. In Section \ref{sec:parameter}, we already observed that $\beta$ and $\gamma$ always get higher weight than $\alpha$ and $\delta$, which indirectly implies that second and third components are important than the other two. We here conduct the following experiment to understand which factor contributes more to the objective function: We drop each component in isolation and modify the additive constraint mentioned in Equation \ref{min1}. For instance, when the fourth component is dropped, the constraint becomes $\frac{\alpha}{2}+\frac{\beta}{2}+\gamma=1$. Then we optimize the objective function and measure the accuracy on different datasets. Table \ref{table:component-wise} shows the percentage decrease in accuracy of \namei~by dropping each component in isolation with respect to the case when all the components are present. We observe that dropping of the second and third components effects the accuracy more compared to the other two. This result once again corroborates with Section \ref{sec:parameter}. Importantly, dropping of any component never increases the accuracy, which implies that all four components need to be considered in the objective function.

\begin{table}[!t]
 \centering
 \caption{Percentage decrease in accuracy (in terms of AUC) of \namei~after removing each of the four components in isolation. Similar pattern is observed for \name~and based on F-Score. Maximum drop is highlighted in bold.}\label{table:component-wise}
 %\vspace{-5mm}
 \scalebox{0.8}{
 \begin{tabular}{|c| l |c| c |c |c |}
 \hline
  & Dataset & - 1st Comp. & - 2nd Comp & - 3rd Comp. & - 4th Comp  \\\hline
  
\multirow{7}{0.03\columnwidth}{\rotatebox[origin=c]{90}{Binary}} 
&   Titanic & 10.87 & 15.49 & {\bf  25.43} & 7.46  \\
 & Spambase & 12.01& {\bf 22.87} & 18.76 & 6.34\\
 & Magic & 12.34& 20.13 & {\bf 23.32} &  21.40   \\
 & Creditcard & 15.09 & {\bf 23.43} & 17.65 & 10.00 \\
% & Wikipedia vandal & 33000 & 2 & 382 & 0.50 & 0.99 & \cite{Kumar:2015}\\
 & Adults & 18.65 & {\bf 22.12} & 20.02 & 9.43 \\
 
 & Diabetes & 12.09 & {\bf 32.41} & 25.17 & 8.98   \\
 
 & Susy &  18.64 & 20.90 & {\bf 24.08} & 8.71  \\\hline				
 
 \multirow{6}{0.05\columnwidth}{\rotatebox[origin=c]{90}{Multi-class}} 
 & Iris &  12.34 & 18.98 & {\bf 23.32} & 5.43 \\
 
 & Image & 12.09 & {\bf 32.98} & 25.29 & 10.08  \\
 
 & Waveform & 18.76 & 24.43 & {\bf 28.87} & 4.56   \\
 
 & Statlog  & 16.33 & 20.08 & {\bf 21.28} & 8.34 \\
 
 &Letter   & 17.87 & 29.87 & {\bf 30.80} & 3.43\\		
 &Sensor  &  18.78 & 25.56 & {\bf 28.09} & 7.34 \\\hline
 \end{tabular} }
%\vspace{-5mm}
\end{table}

\subsection{Handling Class Imbalance}\label{class_imbalance}
As mentioned earlier, \namei~ is specially designed to handle  imbalanced datasets, which other ensemble methods and \name~might not handle well. Among binary and multi-class datasets, Creditcard and Statlog are the most imbalanced ones respectively (see the proportion of majority class MAJ in Table \ref{table:dataset}), and we have already observed in Table \ref{tab:result} that for both these datasets \namei~outpeforms other methods. However, it is not clear how well \namei~can handle even more imbalanced data. Hence we artificially generate imbalanced data from a given dataset as follows. For each dataset, we randomly select one class, and from that class we randomly remove $x\%$ of its constituent objects. The entire process is repeated 10 times for each value of $x$, and the average accuracy  is reported. We vary $x$ from $0\%-30\%$ (with the increment of $5\%$). We conduct this experiment on the largest binary dataset -- Susy, and the largest multi-class dataset -- Sensor because we want to make sure that the change in performance should not be  due to the lack of enough training samples (which might happen if we consider a small dataset), but solely due to the class imbalance problem.  

Figure \ref{fig:balance}(a) shows the average overall AUC (and standard deviation)  of the best baseline method (UPE) and our methods (\name~and \namei) for each value of $x$, i.e., a certain extent of class imbalance. We observe that for both the datasets, UPE is highly sensitive to class imbalance -- the rate of decrease in AUC is significantly higher ($t$-test with 95\% confidence interval) than both of our methods. However, \namei~ is even more effective than \name~ -- after $30\%$ injection of random class imbalance, it is able to retain 87\% and 89\% of its original AUC for Susy and Sensor  respectively. Further investigation on how accurately the competing methods are able to predict the objects of {\em only the manipulated class} reveals the same pattern (see Figure \ref{fig:balance}(b)) -- \namei~ outperforms others in capturing the rare class, followed by \name~ and UPE. With 30\% random class imbalance, \namei, \name~and UPE are able to retain $82\%$, $73\%$ and $67\%$ of its original F-Score for Susy, and $88\%$, $82\%$ and $77\%$ of its original F-Score for Sensor respectively. From these results, we may conclude that irrespective of the proportion of classes in a dataset, \namei~is always effective.

\begin{figure}[!t]
 %\vspace{-2mm}
\centering
 \includegraphics[width=\columnwidth]{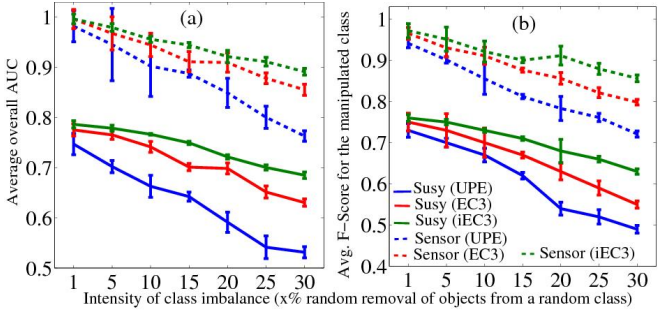}
 \vspace{-5mm}
 \caption{(Color online) Effect of class imbalance on the performance of best baseline method (UPE) and our methods (\name~and \namei) for Susy and Sensor. We randomly remove $x\%$ of objects from  a randomly selected class and measure (a) the  average overall AUC (and SD), and (b) average F-Score (and SD) corresponding to the manipulated class.   }\label{fig:balance}
 %\vspace{-5mm}
 \end{figure}

\begin{figure}[!t]
 \centering
 \includegraphics[width=\columnwidth]{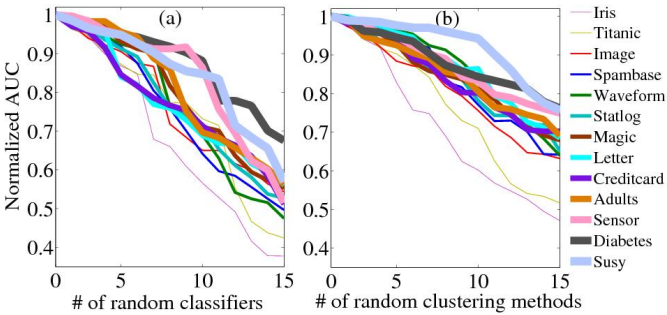}
% \vspace{-3mm}
 \caption{(Color online) Normalized AUC of \namei~after including  (a) $15$ random (a) classifiers and (b) clustering methods in the base set. The width of the line is correlated with the size of the corresponding dataset (more the width, larger the size of the dataset). The datasets are ordered in increasing order of size. }\label{fig:robustness}
\vspace{-5mm}
 \end{figure}

\begin{figure*}[!ht]
\centering
\includegraphics[width=0.9\textwidth]{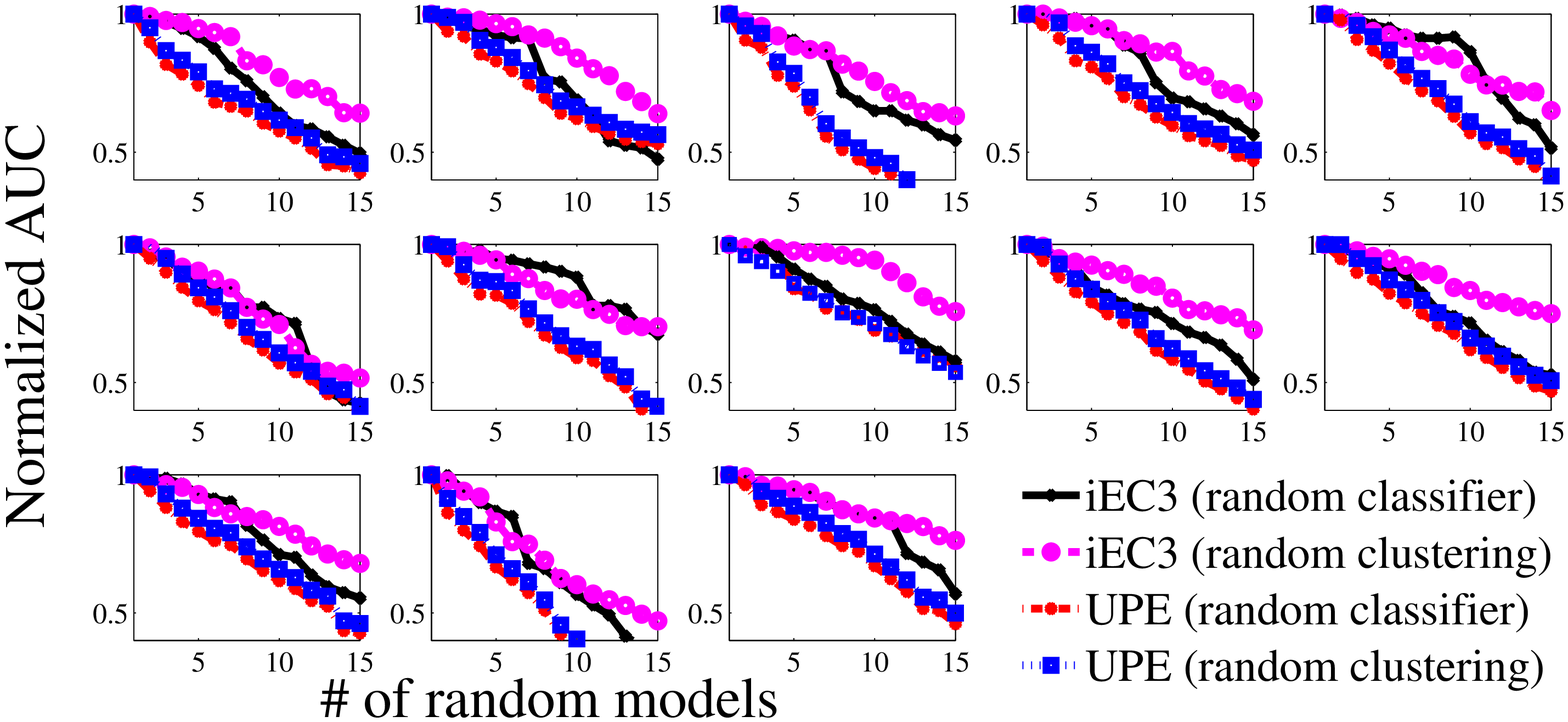}
\caption{Change in (normalized) AUC of \namei~and UPE (best baseline) after including 15 random classifiers and random clustering methods into the base set for 13 datasets (Left to right, top to bottom: Titanic, Spambase, Magic, Creditcard, Adults, Diabetes, Susy, Iris, Image, Waveform, Statlog, Letter, Sensor). The rate of decrease in AUC is high for UPE compared to \namei, which indicate that \namei~is more robust than UPE.}\label{random}
\end{figure*}

%\vspace*{-5mm}
\subsection{Robustness Analysis}\label{robustness}
One might wonder how robust  our method is when random noise is injected into the base set. This might be important in an adversarial setting when attackers constantly try to manipulate the underlying framework to poison base solutions. To check the robustness of our methods, we add multiple randomly generated prediction/clustering into the base set and observe the resilience of our methods to  noise. Two types of random models are developed -- (i) each {\em random classifier} takes an object and randomly assigns a class (from the set of available classes for each dataset)  to it, (ii) each {\em random clustering method} selects a number $c$ between $[1,N]$ uniformly at random (where $N$ and $c$ are the number of objects and clusters respectively) and assigns each object into a cluster randomly with the guarantee that in the end no cluster will remain empty. Figure \ref{fig:robustness} shows the change in accuracy with increase of random base models. We observe that -- (i) \namei~retains at least 88\% of its original performance (noise-less scenario) with 10 random models  incorporated into it, whereas UPE keeps only 58\% of its original accuracy (see Figure \ref{random} for the comparison between \namei~and UPE); (ii) the effect of random classifiers is more detrimental than that of random clustering methods (for each dataset, the lowest value of its corresponding line over Y-axis is lower in Figure \ref{fig:robustness}(a) compared to that in    Figure \ref{fig:robustness}(b)); (iii) small datasets are quickly affected by the noise than large datasets. The first observation indicates that \namei~ is more robust to noise than UPE. The second observation might be explained by the fact that the outputs of the base classifiers are essentially used to determine the final class, whereas base clustering methods only provide an additional constraints. Therefore, noise at classification level harms the final performance more than that at  clustering level. The third observation leads to two conclusions -- first, \namei~ is more robust to large datasets than small datasets; second, to significantly reduce the prediction accuracy  of \namei~ for large datasets, one may really need to infect {\em a lot} of noise into the base set. However, UPE is less robust than \namei~-- the performance of UPE deteriorates even faster than \namei~(see Figure \ref{random}).

  %\vspace{-3mm}
\begin{figure}[!h]
 \centering
 \includegraphics[width=\columnwidth]{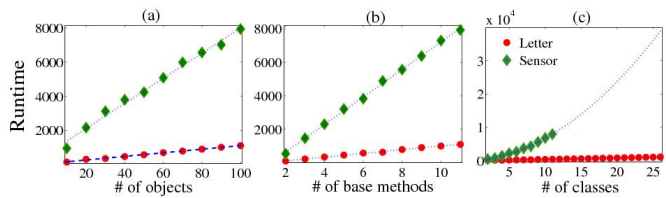}
 \vspace{-5mm}
 \caption{(Color online) Runtime (seconds) of \namei~ with the increase of the number of (a) objects, (b) base methods, and (c) classes for two largest multi-class datasets -- Letter and Sensor (plots are fitted with goodness of fit  $>0.98$ in terms of R-square).  }\label{fig:runtime1}
% \vspace{-6mm}
 \end{figure}

\begin{table}[!h]
 \centering
 \caption{(Color online) Runtime (in seconds) of the ensemble methods that consider both classification and clustering (we do not consider the time to run the base methods).}\label{tab:runtime}
 \vspace{-3mm}
 \scalebox{0.7}{
 \begin{tabular}{|l|r|r|r|r|r|r|r|} 
 %\multicolumn{16}{c}{(a)} & \multicolumn{1}{c}{} & \multicolumn{5}{c}{(b)} & \multicolumn{1}{c}{} & \multicolumn{3}{c}{(c)} \\ 
 %\cline{1-16}\cline{18-22}\cline{24-26}
  
% \multicolumn{7}{c}{Properties of the dataset} & \multicolumn{3}{|c|}{Accuracy ($AUC$)} \\\hline 

 %\multirow{2}{*}{Runtime} & \multicolumn{|c|}{13}%{Dataset} \\\cline{} 
 \multicolumn{8}{c}{(a) Binary Dataset}\\\hline
 Method &Titanic & Spambase & Magic & Creditcard & Adults & Diabetes & Susy  \\\hline 
BGCM &  27 &68  &510 & 1786 & 2440  &5672 & 28109\\
UPE & 31 & 73 & 621 & 1803 & 2519 & 5720 & 28721\\
\namei & 23 & 67 & 436 & 1654 & 2410 & 5478 & 27621 \\\hline

%Method &Titanic & Spambase & Magic & Creditcard & Adults & Diabetes & Susy & Iris & Image & Waveform & Statlog & Letter & Sensor 

%\multicolumn{2}{|c|}{Average} & 8.92 & 11.26 & 8.96 & 9.69 & 9.42& 9.19 & 11.80 & 6.65 & 5.92 & 8.15 & 5.26 & {\color{red}4.07} & 2.03 & {\color {blue}1.07} & & -- & -- & -- & 0.81 & 0.75 & & -- & -- & -- \\\cline{1-16}\cline{18-22} \cline{24-26}	  
 %\multicolumn{16}{c}{} & \multicolumn{1}{c}{} & \multicolumn{9}{c}{Accuracy of UPE (AUC,F-Sc): $^{*}(0.74,0.72)$; $^{\diamond}(0.49,0.03)$;$^{\#}(0.97,0.97)$} \\ 

 \multicolumn{8}{c}{(b) Multi-class Dataset}\\\cline{1-7}
 Method &Iris & Image & Waveform & Statlog & Letter & Sensor  \\\cline{1-7}
BGCM &  20 & 70 & 195 & 345 & 1423 & 7992\\
UPE & 21 & 74 & 208 & 367 & 1567 & 8092\\
\namei & 14 & 68 & 178 & 248 & 1098 & 7934\\\cline{1-7}

%Method &Titanic & Spambase & Magic & Creditcard & Adults & Diabetes & Susy & Iris & Image & Waveform & Statlog & Letter & Sensor 

 \end{tabular} }

 %\vspace{-5mm}
\end{table}

%\vspace{-3mm}
\subsection{Runtime Analysis}\label{runtime}
In Section \ref{algo}, we have mentioned that if we are given the base results a priori, the runtime of our method is linear in the number of objects and the number of base methods, and quadratic in the number of classes. Here we empirically verify our claims on two largest multi-class datasets -- Letter and Sensor, through the following three experiments. (i) we randomly select $10\%$ of total objects per dataset, incrementally add $10\%$ objects in each step and observe that the runtime of \namei~ increases linearly (Figure \ref{fig:runtime1}(a)). (ii) Given the entire dataset, we first add the results of one base classifier and one base clustering method, and then incrementally add remaining $10$ base methods ($6$ classifiers, followed by $4$ clustering methods mentioned in Section \ref{setup}), one per each step and observe that the runtime of \namei~increases linearly (Figure \ref{fig:runtime1}(b)). (iii) For each dataset, we randomly select $2$ classes and the corresponding objects in those classes, and incrementally add other classes one at a time in each step. Since the class-size is unequal, we repeat this experiment $10$ times in each step, and report the average runtime. Figure \ref{fig:runtime1}(c) shows that the runtime is quadratic with the number of classes. Moreover, Table \ref{tab:runtime} reports that  the runtime of \namei~is lowest compared to UPE and BGCM for all the datasets -- on average \namei~is 1.21 ({\em resp. 1.13}) times faster than UPE ({\em resp.} BGCM).

\section{Conclusion}\label{conclusion}
In this paper, we presented \name~and \namei~that take advantage of the complementary constraints provided by multiple classifiers and clustering methods to generate more consolidate results. We showed that the proposed objective function is a convex optimization function. Our theoretical foundation strengthens the utility of the proposed methods. We solved the optimization problem using block coordinate descent method. We further analyzed the optimality and the computational complexity of our method. 

\namei~outperfomed 14 other baselines on each of 13 different datasets, achieving at most $10\%$ higher accuracy than the best baseline. Moreover, it is more efficient than other baselines in terms of handling class imbalance, resilience to random noise and scalability. The issues related to algorithmic parameter selection and choice of appropriate base methods were also studied.

It is still not clear which set of base models we should choose to obtain near-optimal results. It might be possible to retain only the most important base models through a correlation study or a machine learning based approach. We will also aim at interpreting the objective function from other perspectives such as whether it correlates to the PageRank method as mentioned in \cite{NIPS2009_3855}. Another crucial point is how to adopt the model when a few labeled objects are available.  We will publish the code of our proposed methods  upon acceptance of this paper.

\section*{Acknowledgment}
This work was supported in part by the Ramanujan Faculty Fellowship grant. The conclusions and interpretations present in this paper are those of the authors and do not have any relation with the funding agencies. The author would like to thank Prof. V.S. Subrahmanian (Dartmouth College, USA)  for the discussion and effective feedback.

\bibliographystyle{IEEEtran}
%\bibliography{ref,ref1}


\vspace{-12mm}

\begin{IEEEbiography}[{\includegraphics[width=1in,height=1.25in,clip,keepaspectratio]{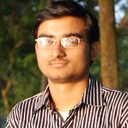}}]{Tanmoy Chakraborty}
is an Assistant Professor and a Ramanujan Fellow in the Dept of Computer Science \& Engineering, IIIT-Delhi, India. Prior to this, he was a postdoctoral researcher at  University of Maryland, College Park, USA.  He finished his Ph.D. as a Google India Ph.D fellow from IIT Kharagpur, India in 2015. His Ph.D thesis was recognized as best thesis by IBM Research India, Xerox research India and Indian National Academy of Engineering (INAE).  
His broad research interests include Data Mining, Social Media and Data-driven Cybersecurity. 
%Home page: \url{http://faculty.iiitd.ac.in/~tanmoy/}
\end{IEEEbiography}

\if{0}
\begin{IEEEbiographynophoto}{Vijay Balakrishnan}
is a student in the Dept. of Computer Science, University of Maryland, College Park, USA. His research areas are data mining and social network analysis.
\end{IEEEbiographynophoto}
\fi


%\vfill


\end{document}